\documentclass{article}
\usepackage{iclr2027_conference,times}
\iclrfinalcopy
\usepackage{amsmath,amsfonts,bm}

\def\eqref#1{equation~\ref{#1}}

\def\1{\bm{1}}

\DeclareMathAlphabet{\mathsfit}{\encodingdefault}{\sfdefault}{m}{sl}
\SetMathAlphabet{\mathsfit}{bold}{\encodingdefault}{\sfdefault}{bx}{n}

\newcommand{\Cov}{\mathrm{Cov}}

\usepackage{amsmath,amssymb,amsthm,mathtools}
\usepackage{graphicx}
\usepackage{subcaption}
\usepackage{float}
\usepackage{wrapfig}
\usepackage{booktabs}
\usepackage{tabularx}
\usepackage{xcolor}
\usepackage{microtype}
\usepackage{enumitem}
\usepackage{needspace}
\usepackage{hyperref}
\usepackage{url}
\usepackage{bbm}
\hypersetup{hidelinks}
\usepackage{multirow}
\newcommand{\method}{\textsc{Suture}}

\newcommand{\ind}{\mathbf{1}}

\newtheorem{proposition}{Proposition}
\newtheorem{corollary}{Corollary}
\newtheorem{lemma}{Lemma}
\theoremstyle{remark}

\theoremstyle{plain}

\title{Beyond Scalar IoU: Structured Verification from Rollout Groups for Video Temporal Grounding}

\author{
  Youngjae Cho \\
  PYLER \\
  \And
  Won Young Jhoo \\
  PYLER \\
  \And
  Jongsuk Kim \\
  PYLER
}
\begin{document}
\maketitle

\begin{abstract}

Reinforcement learning with verifiable rewards (RLVR) provides a natural
framework for adapting pretrained models to video temporal grounding, where generated
temporal intervals can be scored directly against ground truth intervals.
Yet existing overlap verifiers typically score each rollout independently, leaving the joint structure of the rollout group unused.
We introduce SUTURE, which conditions verification on the rollout group and exploits its
structure at two complementary scales: disagreement across rollouts controls how
strongly the target is reweighted, while coverage at each position determines
where reward mass is redistributed. We show that the resulting verifier admits an exact
decomposition into the standard IoU term and a covariance correction determined by the
rollout group.
A local gradient diagnostic finds a preference for responses covering
relatively less supported target regions in the analyzed groups.
Across five temporal grounding benchmarks, SUTURE improves grounding
performance at every reported IoU threshold. Its trained policy also shows
less video-start anchoring in reasoning traces: for later events, the first
temporal mention more often overlaps the annotated target.
Together, these results show that the joint structure of a rollout group can support a more informative temporal verifier.

\end{abstract}

\section{Introduction}
\label{sec:intro}

Video temporal grounding requires a model to identify when the event
described by a natural language query occurs
\citep{gao2017tall,krishna2017dense}.
Recent multimodal large language models (MLLMs) tackle this task by
directly generating temporal intervals from video--query inputs
\citep{bai2025qwen25vltechnicalreport,yan2025videochatr}.
These predictions can be scored against ground truth intervals, making
reinforcement learning with verifiable rewards (RLVR) a natural
framework for adapting pretrained models \citep{shao2024deepseekmath,wang2025timer}.
The verifier determines how localization accuracy becomes
a learning signal.

Group-relative optimization samples multiple responses for the same
input and normalizes their rewards within the rollout group
\citep{shao2024deepseekmath}.
Yet temporal overlap verifiers typically score each interval
independently against the target interval using an overlap reward based
on intersection over union (IoU)
\citep{wang2025timer,zheng2026temporalaware}.
\emph{Optimization is group-relative, but overlap verification
remains individual.}
The group is used to compare scalar rewards, while the temporal
support pattern revealed by its predictions is left out of the
overlap verifier.

Figure~\ref{fig:main} (left) illustrates the resulting distinction.
Two predictions can receive the same IoU while covering different
parts of the target interval: one reaches a region covered by most
rollouts, whereas the other reaches an equally large correct region
covered by only a few.
Equal accuracy under IoU need not imply equal training emphasis:
the latter captures correct coverage that is less consistently
represented in the group.
We use this distinction to emphasize relatively less supported
correct regions during training, rather than treating all correctly
covered locations uniformly.

We introduce Structured Under-support Targeting Using Rollout
Evidence (\method{}), a verifier that conditions on the rollout group and separates
\emph{where} to redistribute target weight from \emph{how much} to
redistribute (Figure~\ref{fig:main}).
Among target locations covered by at least one rollout, we use coverage at each
position to give greater relative weight to less supported regions.
Pairwise disagreement between the complete predicted intervals
controls the redistribution strength.
Normalization preserves total target mass, while the IoU
denominator remains unchanged.
The annotation remains the sole source of correctness:
the group changes only the relative emphasis within it.

We show that this verifier decomposes exactly into ordinary IoU
and a covariance correction determined by the rollout group.
This correction can favor predictions covering relatively
less supported target regions even among predictions with the
same IoU.
A separate local gradient diagnostic finds a corresponding preference
for responses covering relatively less supported target regions
in the analyzed groups.

Across five benchmarks covering both short and long videos, \method{} consistently improves temporal grounding performance.
Component ablations show that structured target weighting improves grounding performance. Adapting the reweighting strength to rollout disagreement yields further gains over the tested controls with fixed strength.
Additional analyses show reduced anchoring to the video start in reasoning traces and shorter responses on average.

\begin{figure}[t!]
    \centering
    \includegraphics[width=\linewidth]{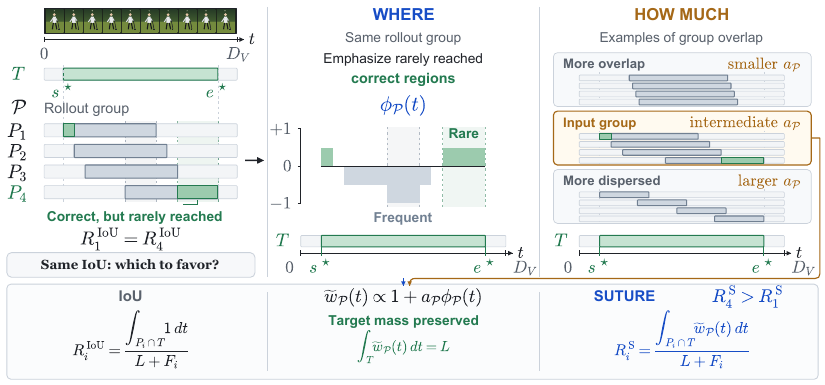}
    \caption{
    \textbf{Target weighting based on rollout groups in \method{}.}
    \textbf{Left:} $P_1$ and $P_4$ have equal IoU but different support over
    their correct overlap.
    \textbf{Center:} The signed residual $\phi_{\mathcal P}(t)$ determines
    where to redistribute target weight.
    \textbf{Right:} Disagreement controls $a_{\mathcal P}$, subject to the floor.
    \textbf{Bottom:} Normalization preserves target mass $L$ and the IoU
    denominator, while the structured weighting yields
    $R_4^{\mathrm S}>R_1^{\mathrm S}$ in this example.
    }
    \label{fig:main}
\end{figure}

Our main contributions are summarized as follows:
\begin{itemize}
    \item \textbf{Verification using rollout groups for temporal grounding.}
    We introduce \method{}, which uses rollout coverage at each position to determine
    where to redistribute target weight and disagreement across the group to set the
    amplitude, while retaining the
    annotation as the source of correctness.
    \item \textbf{Exact characterization and local analysis.}
    We derive an exact reward decomposition and ordering by group support, and
    separately examine the structured residual through
    a local gradient diagnostic.
    \item \textbf{Empirical validation across models, benchmarks,
    and temporal behavior.}
    We demonstrate grounding gains across five benchmarks,
    evaluate the amplitude rule based on disagreement against controls with fixed
    amplitudes, and observe reduced anchoring to the video start in
    reasoning traces.
\end{itemize}

\section{Preliminaries}
\label{sec:prelim}

We introduce the temporal grounding setup, standard overlap verification, and the group-relative optimization quantities used throughout the paper.

\subsection{Temporal Grounding and Overlap Verification}
\label{sec:prelim_grounding}

For a video--query input $x=(V,q)$ with video duration $D_V$, let the ground truth interval be
\[
    \mathcal T=[s^\star,e^\star],
    \qquad
    L=|\mathcal T|=e^\star-s^\star>0.
\]
Here $|A|$ denotes the total duration of a temporal set $A$, so $L$ is the target duration. The policy model $\pi_\theta$ generates a group of $G$ responses.
During rollout collection, responses are sampled from the behavior policy $\pi_{\theta_{\mathrm{old}}}$ as
\begin{equation}
    y_i \sim \pi_{\theta_{\mathrm{old}}}(\cdot\mid x),
    \qquad
    i=1,\ldots,G.
    \label{eq:rollout_group}
\end{equation}
A valid prediction in response $y_i$ defines $P_i=[\hat s_i,\hat e_i]$
with $\hat s_i< \hat e_i$. 
We write
\[
    \mathcal P=(P_1,\ldots,P_G)
\]
for the realized rollout tuple. Throughout, $\mathcal P$ denotes the particular group sampled for the current input and is an ordered tuple rather than a distribution.

For rollout $i$, define its duration of correct overlap and its duration outside the target as
\begin{equation}
\label{eq:IF}
    I_i:=I(P_i)=|P_i\cap\mathcal T|,
    \qquad
    F_i:=F(P_i)=|P_i\setminus\mathcal T|.
\end{equation}
The union $P_i\cup\mathcal T$ consists of the full target $\mathcal T$ and the
predicted portion outside it, $P_i\setminus\mathcal T$. These parts are disjoint, so $|P_i\cup\mathcal T|
    =|\mathcal T|+|P_i\setminus\mathcal T|
    =L+F_i$. 
The standard temporal IoU reward is therefore
\begin{equation}
\label{eq:iou}
    R_i^{\mathrm{IoU}}
    :=
    R^{\mathrm{IoU}}(P_i,\mathcal T)
    =
    \frac{|P_i\cap\mathcal T|}{|P_i\cup\mathcal T|}
    =
    \frac{I_i}{L+F_i}
    =
    \frac{
        \displaystyle\int_{P_i\cap\mathcal T}1\,dt
    }{
        L+F_i
    }.
\end{equation}
For fixed $(P_i,\mathcal T)$, $R_i^{\mathrm{IoU}}$ does not depend on the other rollouts in $\mathcal P$.

\subsection{Group-Relative Policy Optimization}
\label{sec:prelim_grpo}

We use uppercase $R$ for individual reward terms and lowercase $r$ for
the total training reward used to compute the GRPO advantage.
Given rollout rewards $\{r_i\}_{i=1}^{G}$,
GRPO~\citep{shao2024deepseekmath} forms advantages normalized within each
group
\begin{equation}
    \bar r=\frac{1}{G}\sum_{i=1}^{G}r_i,
    \qquad
    \sigma_r=
    \sqrt{\frac{1}{G}\sum_{i=1}^{G}(r_i-\bar r)^2},
\end{equation}
\begin{equation}
\label{eq:grpo_advantage}
    A_i
    =
    \operatorname{sg}\!\left[
        \frac{r_i-\bar r}{\sigma_r+\delta}
    \right],
    \qquad \delta>0,
\end{equation}
where $\operatorname{sg}[\cdot]$ denotes stop-gradient. The full clipped GRPO surrogate is given in Appendix~\ref{app:full_grpo}; the main text retains only the quantities needed for our verifier analysis.

\section{Method}
\label{sec:method}

\subsection{Target weighting from rollout groups}
\label{sec:two_scale}

For rollout $P_i$, \method{} computes the overlap reward
$R_i^{\mathrm S}:=R^{\mathrm S}(P_i,\mathcal T;\mathcal P)$ by
replacing the uniform weighting of the target in IoU:
\begin{equation}
\label{eq:suture_reward}
    R_i^{\mathrm{IoU}}
    =\frac{\displaystyle\int_{P_i\cap\mathcal T}1\,dt}{L+F_i},
    \qquad
    R_i^{\mathrm S}
    =\frac{\displaystyle\int_{P_i\cap\mathcal T}
        \widetilde w_{\mathcal P}(t)\,dt}{L+F_i}.
\end{equation}
IoU counts every correctly covered temporal location equally;
\method{} weights these locations using rollout group support.
The shared weight field $\widetilde w_{\mathcal P}$ is obtained by normalizing $1+a_{\mathcal P}\phi_{\mathcal P}(t)$ over the target to preserve total mass $L$.
The support residual $\phi_{\mathcal P}(t)$ determines the reweighting pattern from coverage at each position, while the amplitude $a_{\mathcal P}$ controls its strength based on rollout disagreement.
Each rollout integrates this field over its own correct overlap $P_i\cap\mathcal T$, with the IoU denominator unchanged.

\paragraph{Where to reweight: group support at each position.}
The fraction of rollouts covering target location $t$ is
\begin{equation}
\label{eq:coverage}
    C_{\mathcal P}(t)
    =
    \frac{1}{G}
    \sum_{i=1}^{G}\mathbbm{1}[t\in P_i],
    \qquad t\in\mathcal T.
\end{equation}

We contrast the fraction of rollouts that do not cover a target
location with the fraction that do:
$(1-C_{\mathcal P}(t))-C_{\mathcal P}(t)
=1-2C_{\mathcal P}(t)$.
We apply this contrast only at target locations covered by at
least one rollout, defining the support residual as
\begin{equation}
\label{eq:residual}
    \phi_{\mathcal P}(t)
    =
    \mathbbm{1}[C_{\mathcal P}(t)>0]\,
    \bigl(1-2C_{\mathcal P}(t)\bigr),
\end{equation}
where $\mathbbm{1}(\cdot)$ is indicator function. Among covered target locations, lower coverage gives a larger
residual.
An uncovered location contributes to no member's correct overlap.
We leave its residual at zero, retaining a raw weight of one
before normalization of target mass.

\paragraph{How strongly to reweight: rollout disagreement.}
We measure disagreement using overlap between complete predicted
intervals. Their pairwise interval IoU is $\operatorname{pIoU}(P_i,P_j) = \frac{|P_i\cap P_j|}{|P_i\cup P_j|}$.
Averaging over pairs gives one agreement value for the entire group:
\begin{equation}
\label{eq:avg_pair_iou}
    \operatorname{AvgPairIoU}(\mathcal P)
    =
    \frac{2}{G(G-1)}
    \sum_{1\le i<j\le G}
    \operatorname{pIoU}(P_i,P_j).
\end{equation}
The raw disagreement and its floored redistribution amplitude are
\begin{equation}
\label{eq:amplitude}
    u_{\mathcal P}
    :=1-\operatorname{AvgPairIoU}(\mathcal P),
    \qquad
    a_{\mathcal P}
    =\max\{u_{\mathcal P},u_{\mathrm{floor}}\}.
\end{equation}
The amplitude follows $u_{\mathcal P}$ above the floor and otherwise remains
at $u_{\mathrm{floor}}$.

\paragraph{Combining location and strength.}
We normalize the combined field $1+a_{\mathcal P}\phi_{\mathcal P}(t)$
over the target:
\begin{equation}
\label{eq:niw_weight}
    \widetilde w_{\mathcal P}(t)
    =
    \frac{
        L\bigl[1+a_{\mathcal P}\phi_{\mathcal P}(t)\bigr]
    }{
        \displaystyle\int_{\mathcal T}
        \bigl[1+a_{\mathcal P}\phi_{\mathcal P}(t')\bigr]\,dt'
    },
    \qquad
    \int_{\mathcal T}\widetilde w_{\mathcal P}(t)\,dt
    =
    L.
\end{equation}
This redistributes weight within the target while keeping its total mass at $L$.
Without the floor, a small $u_{\mathcal P}$ would make
$\widetilde w_{\mathcal P}$ close to uniform IoU weighting even if
$\phi_{\mathcal P}$ still varies across the target. The floor keeps a minimum
amplitude for this remaining difference in support.
Throughout, we take $G\ge2$ and $0\le u_{\mathrm{floor}}<1$. On this constraint,
$1+a_{\mathcal P}\phi_{\mathcal P}(t)$ is strictly positive within each
target segment of nonzero duration between consecutive interval boundaries
(proved in Lemma~\ref{lem:positivity}), so the normalization is well defined.

This completes the construction of the overlap reward conditioned on the rollout group
in Equation~\ref{eq:suture_reward}. The next section characterizes how
this reward differs from ordinary IoU and how the difference enters
the group-relative update.

\subsection{Theoretical Characterization}
\label{sec:theory}

We first characterize how group support changes each prediction's
reward relative to IoU. We then examine how this reward correction
enters the group-relative gradient.
Full proofs and additional properties are provided in
Appendices~\ref{app:theory_characterization} and~\ref{app:proofs}.

\paragraph{Exact reward correction.}
Let $U\sim\mathrm{Unif}(\mathcal T)$ denote a uniformly distributed
temporal location within the target.

\begin{proposition}[Exact refinement using rollout groups]
\label{prop:refinement}
For a fixed realized rollout group $\mathcal P$, let
\[
    \bar\phi_{\mathcal P}
    =
    \frac{1}{L}\int_{\mathcal T}\phi_{\mathcal P}(t)\,dt,
    \qquad
    \beta_{\mathcal P}
    =
    \frac{a_{\mathcal P}}
         {1+a_{\mathcal P}\bar\phi_{\mathcal P}}
    \ge 0,
\]
and, for member $P_i$, define
\begin{equation}
\label{eq:xi_member}
    \xi_i
    :=
    \frac{
        L\,\Cov\!\left(
            \phi_{\mathcal P}(U),\ind[U\in P_i]
        \right)
    }{
        L+F_i
    }.
\end{equation}
Then the normalized target weight field satisfies
\begin{equation}
\label{eq:pointwise_refinement}
    \widetilde w_{\mathcal P}(t)-1
    =
    \beta_{\mathcal P}
    \bigl(\phi_{\mathcal P}(t)-\bar\phi_{\mathcal P}\bigr),
\end{equation}
and the verifier admits the exact decomposition
\begin{equation}
\label{eq:main_refinement}
    R_i^{\mathrm S}
    =
    R_i^{\mathrm{IoU}}
    +
    \beta_{\mathcal P}\xi_i.
\end{equation}
\end{proposition}

Normalization of target mass centers the support residual around its
mean over the target (Equation~\ref{eq:pointwise_refinement}). The
coefficient $\beta_{\mathcal P}$ is determined by the amplitude and
normalization; it introduces no additional parameter. The covariance in
Equation~\ref{eq:xi_member} compares the residual with individual
coverage over temporal locations in $\mathcal T$, holding the group fixed.
For $I_i>0$, it is positive precisely when the mean residual over
$P_i\cap\mathcal T$ exceeds $\bar\phi_{\mathcal P}$.

\paragraph{Ordering by group support.}
Consider two members of the same realized group with
$a_{\mathcal P}>0$, equal positive overlap durations within the target, and equal
durations outside it. Their IoU scores are equal, but
\method{} gives a higher structured overlap reward to the member whose correct overlap
has lower average group coverage
(Appendix~\ref{app:theory_characterization}). The coefficient
$\beta_{\mathcal P}$ scales this difference.

For a fixed residual field, $\beta_{\mathcal P}$ increases with
the amplitude $a_{\mathcal P}$, since
$\partial\beta_{\mathcal P}/\partial a_{\mathcal P}
=(1+a_{\mathcal P}\bar\phi_{\mathcal P})^{-2}>0$.
A smaller $a_{\mathcal P}$ therefore reduces the correction relative
to IoU, recovered at $a_{\mathcal P}=0$.
When IoU and the support correction favor different members,
$\beta_{\mathcal P}$ determines their relative contributions to the
reward difference.

\paragraph{Effect on the GRPO gradient.}
We next examine whether the reward correction only rescales the
IoU gradient or can also change its direction.
For fixed rollouts and policy parameters, the verifier affects
the advantages normalized within each group, not the rollout score vectors.
At the behavior policy, each score vector is the gradient of the
log-likelihood normalized by response length:
\begin{equation}
    \mathbf z_i
    =
    \frac{1}{|y_i|}
    \sum_{t=1}^{|y_i|}
    \left.
    \nabla_{\theta}
    \log
    \pi_{\theta}(y_{i,t}\mid x,y_{i,<t})
    \right|_{\theta=\theta_{\mathrm{old}}},
    \label{eq:rollout_score}
\end{equation}
With clipping locally inactive, the realized direction determined by the reward is
\begin{equation}
    \mathbf d(r)
    =
    \frac{1}{G}
    \sum_{i=1}^{G}
    A_i\mathbf z_i.
    \label{eq:local_reward_update}
\end{equation}
We omit the KL regularizer, which is unchanged by the verifier.

The total training rewards $r_i^{\mathrm{IoU}}$ and $r_i^{\mathrm S}$
combine the respective temporal scores with the same format reward.
For the gradient analysis, we write their difference as
$\chi_i:=r_i^{\mathrm S}-r_i^{\mathrm{IoU}}
=\beta_{\mathcal P}\xi_i$
by Proposition~\ref{prop:refinement}.
Define the group mean
$\bar\chi=G^{-1}\sum_{i=1}^{G}\chi_i$
and centered correction
$\widetilde\chi_i=\chi_i-\bar\chi$.

\begin{samepage}
Let
$\mathbf d_{\mathrm S}:=\mathbf d(r^{\mathrm S})$ and
$\mathbf d_{\mathrm{IoU}}:=\mathbf d(r^{\mathrm{IoU}})$.
At the behavior policy, with locally inactive clipping,
stop-gradient advantages, and $\delta>0$, Equation~\ref{eq:local_reward_update} gives
\begin{equation}
\label{eq:gradient_decomp_body}
    \mathbf d_{\mathrm S}
    =
    \kappa_{\mathcal P}\mathbf d_{\mathrm{IoU}}
    +
    \alpha_{\mathcal P}\mathbf d_{\mathrm{corr}},
    \qquad
    \mathbf d_{\mathrm{corr}}
    =
    \frac{1}{G}\sum_{i=1}^{G}
    \widetilde\chi_i\mathbf z_i,
\end{equation}
with
\[
    \kappa_{\mathcal P}
    =
    \frac{\sigma_{\mathrm{IoU}}+\delta}
         {\sigma_{\mathrm S}+\delta},
    \qquad
    \alpha_{\mathcal P}
    =
    \frac{1}{\sigma_{\mathrm S}+\delta}>0.
\]
Here $\sigma_{\mathrm{IoU}}$ and $\sigma_{\mathrm S}$ are the
standard deviations of the corresponding total training rewards within the group. Thus
$\kappa_{\mathcal P}$ rescales the IoU gradient to
reflect the change in reward standard deviation.
The centered correction contributes an additional component that
can change the gradient direction.
\end{samepage}

Section~\ref{sec:amplitude_ablation} compares fixed and adaptive
amplitudes; Section~\ref{sec:update_geometry} probes local response
log-likelihood changes with rollouts and model parameters held fixed.

\section{Experiments}
\label{sec:experiments}

\subsection{Experimental settings}

\paragraph{Benchmarks.}
We evaluate on five temporal grounding benchmarks.
QVHighlights \citep{lei2021detecting} pairs vlog and news videos
with natural language queries and annotated relevant moments.
ActivityNet Captions \citep{krishna2017dense} provides event
descriptions aligned with temporal intervals, while Charades-STA
\citep{gao2017tall} pairs activity descriptions with their
corresponding video segments.
We additionally evaluate on Ego4D-NLQ \citep{Grauman_2022_CVPR},
which requires locating segments that answer questions about
the camera wearer's experience, and TACoS
\citep{regneri-etal-2013-grounding}, which aligns descriptions of
cooking actions with video segments.
\paragraph{Evaluation metrics.}
We report top-1 recall at IoU thresholds $m\in\{0.3,0.5,0.7\}$,
denoted by R1@$m$.
This metric measures the percentage of queries whose top-ranked
predicted interval achieves an IoU above $m$ with the annotated target.
\paragraph{Training setup.}
The controlled verifier experiments and TaRO reproduction use the training
framework of TaRO \citep{zheng2026temporalaware}. We reproduce Time-R1 separately
following its original training settings \citep{wang2025timer}.
The primary experiments use Qwen2.5-VL-7B-Instruct
\citep{bai2025qwen25vltechnicalreport} with $G=8$ rollouts per prompt;
\method{} uses $u_{\mathrm{floor}}=0.3$.
We additionally compare TaRO and \method{} with
Qwen3-VL-8B-Instruct \citep{bai2025qwen3vltechnicalreport}.
Training and evaluation details are provided in
Appendix~\ref{app:training_details}, and sensitivity to group size
is reported in Appendix~\ref{app:group_size}.
\paragraph{Reward variants.}
The Naive IoU control, fixed-amplitude controls, and \method{} use
the same training settings and binary format reward, differing only
in the overlap verifier.
$R_i^{\mathrm{format}}$ is 1 if the response consists of one
\texttt{<think>} block followed by one \texttt{<answer>} block and 0 otherwise;
it checks structure only.
The Naive IoU control uses
$r_i^{\mathrm{IoU}} = R_i^{\mathrm{format}} + R_i^{\mathrm{IoU}}$.
For context, the reward formulations of
Time-R1~\citep{wang2025timer}, TaRO, and SUTURE are
\begin{equation}
\begin{aligned}
r_i^{\mathrm{Time\text{-}R1}}
&=
R_i^{\mathrm{format}}
+
R_i^{\mathrm{tIoU}},
\\
r_i^{\mathrm{TaRO}}
&=
R_i^{\mathrm{format}}
+
R_i^{\mathrm{tIoU}}
+
R_i^{\mathrm{temp}}
\ind[\operatorname{IoU}_i > \tau],
\\
r_i^{\mathrm S}
&=
R_i^{\mathrm{format}}
+
R_i^{\mathrm S}.
\end{aligned}
\end{equation}
Here, $R_i^{\mathrm{tIoU}}$ denotes the timestamp-aware IoU reward used by
Time-R1 and TaRO, and $R_i^{\mathrm{temp}}$ denotes TaRO's temporal-reasoning reward.
For the fixed-amplitude controls, we replace $a_{\mathcal P}$ in
$R_i^{\mathrm S}$ with a constant while keeping the remaining verifier construction unchanged.
Definitions of the Time-R1 and TaRO reward terms are provided in
Appendix~\ref{sec:related_work}.

\begin{table}[h]
    \caption{\textbf{Direct comparison across two MLLM backbones.} All numbers are
    R1 at the stated IoU threshold. The broader baseline comparison is provided in
    Appendix~\ref{app:full_benchmark_results}.}
    \label{tab:main_results}
    \centering
    \small
        \resizebox{\linewidth}{!}{%
        \begin{tabular}{lccccccccc}
        \toprule
        & \multicolumn{3}{c}{QVHighlights} & \multicolumn{3}{c}{ActivityNet} & \multicolumn{3}{c}{Charades-STA} \\
        \cmidrule(lr){2-4}\cmidrule(lr){5-7}\cmidrule(lr){8-10}
        Method & R1@0.3 & R1@0.5 & R1@0.7 & R1@0.3 & R1@0.5 & R1@0.7 & R1@0.3 & R1@0.5 & R1@0.7 \\
        \midrule
        \multicolumn{10}{l}{\textit{Qwen2.5-VL-7B-Instruct as base model}} \\
        Qwen2.5-VL-7B-Instruct & 15.9 & 7.1  & 4.2  & 24.4 & 13.6 & 6.7  & 72.5 & 53.6 & 28.5 \\
        VideoChat-R1.5   & 71.4  & 55.8 & 38.4 & 52.4 & 32.3 & 16.8 & --   & --   & --   \\
        Time-R1          & 80.3  & 66.2 & 44.8 & 58.6 & 39.0 & 21.4 & 78.1 & 60.8 & 35.3 \\
        TaRO             & 81.7  & 66.8 & 45.3 & 61.0 & 39.7 & 20.5 & 79.3 & 63.3 & 36.7 \\
        \method{} (ours) & \textbf{84.1} & \textbf{70.1} & \textbf{50.6} & \textbf{66.8} & \textbf{47.1} & \textbf{26.2} & \textbf{81.6} & \textbf{65.1} & \textbf{37.6} \\
        \midrule
        \multicolumn{10}{l}{\textit{Qwen3-VL-8B-Instruct as base model}} \\
        TaRO             & 83.2 & 68.7 & 51.7 & 58.0 & 39.8 & 23.9 & 82.9 & 67.0 & 37.7 \\
        \method{} (ours) & \textbf{84.3} & \textbf{70.0} & \textbf{54.4} & \textbf{63.0} & \textbf{44.5} & \textbf{27.3} & \textbf{83.6} & \textbf{67.4} & \textbf{39.9} \\
        \bottomrule
        \end{tabular}}
\end{table}
\subsection{Temporal grounding results}
Table~\ref{tab:main_results} presents the results on
QVHighlights \citep{lei2021detecting}, ActivityNet Captions
\citep{krishna2017dense}, and Charades-STA \citep{gao2017tall}.
We omit VideoChat-R1.5 \citep{yan2025videochatr} results on
Charades-STA because Charades data were included in its training. \method{} outperforms previous methods at every reported IoU threshold on all three benchmarks with both backbones.
The largest absolute gains occur on ActivityNet Captions,
where R1@0.5 improves by $7.4$ and $4.7$ percentage points
with Qwen2.5-VL-7B-Instruct and Qwen3-VL-8B-Instruct,
respectively.
The gains at R1@0.7 indicate that more predictions achieve
close temporal overlap with the annotated event, rather than
the improvement being confined to coarse localization.
A separate three-run comparison also yields higher mean
mIoU and R1 at all three thresholds on these benchmarks
(Appendix~\ref{app:training_runs}).

\begingroup
\setlength{\intextsep}{6pt}
\begin{wraptable}{r}{0.5\textwidth}
    \captionsetup{font=small,skip=4pt}
    \caption{\textbf{Temporal grounding in long videos.}
    Both methods use Qwen3-VL-8B-Instruct as the base model.
    Values are R1 at the stated IoU thresholds.}
    \label{tab:long_video_results}
    \centering
    \small
    \setlength{\tabcolsep}{3pt}
    \resizebox{\linewidth}{!}{
    \begin{tabular}{llccc}
        \toprule
        Dataset & Method & R1@0.3 & R1@0.5 & R1@0.7 \\
        \midrule
        \multirow{2}{*}{Ego4D-NLQ}
            & TaRO & 13.0 & 7.7 & 3.8 \\
            & \method{} (ours)
            & \textbf{13.5} & \textbf{8.1} & \textbf{4.3} \\
        \midrule
        \multirow{2}{*}{TACoS}
            & TaRO & 27.4 & 17.3 & 7.2 \\
            & \method{} (ours)
            & \textbf{28.4} & \textbf{17.7} & \textbf{7.6} \\
        \bottomrule
    \end{tabular}}
\end{wraptable}

\paragraph{Benchmarks with long videos.}
We further examine whether the benefits of \method{} extend
to long-video grounding on Ego4D-NLQ
\citep{Grauman_2022_CVPR} and TACoS
\citep{regneri-etal-2013-grounding}.
Both methods use Qwen3-VL-8B-Instruct as the base model.
As shown in Table~\ref{tab:long_video_results}, \method{}
achieves higher recall than TaRO at all three IoU thresholds
on both benchmarks.
Although the absolute gains are modest, these results support
the applicability of \method{} to the evaluated long-video
settings.
\par\WFclear
\endgroup

\raggedbottom
\subsection{Component analysis of verification using rollout groups}
\label{sec:amplitude_ablation}

We examine two design choices: weighting target regions using
rollout-group support, and adapting the strength of this weighting
to rollout disagreement.
Table~\ref{tab:amplitude_ablation} compares Naive IoU,
fixed-amplitude structured verification, and adaptive \method{}
under the same training settings.
Results are averaged equally over QVHighlights and Charades-STA.

\paragraph{Target weighting based on rollout groups.}
The fixed controls retain target weighting based on rollout-group
support but apply the same amplitude to every group.
The weighting pattern therefore still depends on the sampled
predictions, even though its amplitude is constant.
All tested constants improve over Naive IoU on the four
benchmark-averaged metrics.
These results support the utility of group-conditioned target
weighting even without adapting its amplitude to disagreement.

\paragraph{Adapting the amplitude to disagreement.}
We then compare the adaptive rule
$a_{\mathcal P}=\max\{u_{\mathcal P},0.3\}$
with fixed-amplitude controls at $a\in\{0.3,0.4,0.8\}$,
retaining the same rules for constructing support and
normalizing target mass.
The $a=0.3$ control matches the default lower bound
$u_{\mathrm{floor}}=0.3$, while $a=0.4$ matches the observed
mean of the adaptive amplitude $a_{\mathcal P}$ during training. We also include $a=0.8$, selected for its strong performance
among the tested fixed-amplitude settings.
Appendix~\ref{app:component_behavior} uses the mean-matched $a=0.4$ control
for the behavioral comparison.
\method{} achieves higher scores than all three fixed controls
on the four benchmark averaged metrics.
Neither matching the observed mean amplitude nor using the
larger constant $a=0.8$ reproduces the adaptive result.
This supports using rollout disagreement to set the amplitude
rather than applying one constant across groups.

\begin{table}[h]
\caption{\textbf{Verifier and amplitude ablation.}
Values are averages with equal weight over QVHighlights and Charades-STA,
reported as percentages rounded to one decimal place.}
\label{tab:amplitude_ablation}
\centering
\small
\begin{tabular}{lcccc}
\toprule
Verifier / amplitude rule & mIoU & R1@0.3 & R1@0.5 & R1@0.7 \\
\midrule
Naive IoU (uniform target) & 54.8 & 78.6 & 61.9 & 37.2 \\
Fixed $a=0.3$ (floor-matched) & 56.1 & 79.1 & 63.3 & 39.4 \\
Fixed $a=0.4$ (mean-matched)  & 57.3 & 80.9 & 65.4 & 41.2 \\
Fixed $a=0.8$        & 57.5 & 81.1 & 65.3 & 41.6 \\
\textbf{Adaptive \method{}} & \textbf{59.2} & \textbf{82.9} & \textbf{67.6} & \textbf{44.1} \\
\bottomrule
\end{tabular}
\end{table}

Benchmark-level ablation results and sensitivity to the
disagreement floor are provided in
Appendices~\ref{app:constant_amplitude} and~\ref{app:ufloor}.
Furthermore, Appendix~\ref{app:component_behavior} examines how the gains from fixed and
adaptive verification vary with event start time.

\subsection{Temporal behavior and generated response length}
\label{sec:onset_tracking}

\begingroup
\setlength{\intextsep}{6pt}
\setlength{\textfloatsep}{12pt}
\setlength{\floatsep}{8pt}
\captionsetup[figure]{font=small,skip=4pt}

\paragraph{Tracking event start times.}
The aggregate grounding results do not reveal how the learned policies differ
in their temporal behavior. We therefore examine the first well-formed
temporal span in the reasoning trace before the final answer
(the first temporal mention) and compare its location with the annotated target.
Figure~\ref{fig:onset_tracking} compares \method{}, TaRO, and Naive IoU.

\begin{figure}[h]
    \centering
    \begin{subfigure}[t]{0.34\linewidth}
        \centering
        \includegraphics[width=\linewidth]{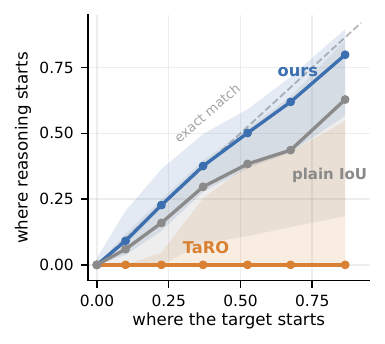}
        \caption{Tracking event start time.}
    \end{subfigure}\hspace{0.07\linewidth}
    \begin{subfigure}[t]{0.34\linewidth}
        \centering
        \includegraphics[width=\linewidth]{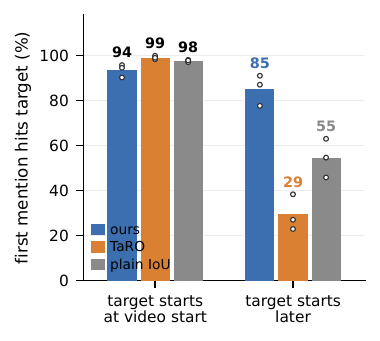}
        \caption{First-mention target-hit rate.}
    \end{subfigure}
    \caption{\textbf{Temporal behavior across trained policies.}
    \textbf{(a)} On ActivityNet, the normalized start of the first temporal
    mention is plotted against the normalized event start time; the dashed
    diagonal denotes exact tracking.
    \textbf{(b)} First-mention target-hit rates for events starting at the video
    start or later. Bars average QVHighlights, ActivityNet Captions, and
    Charades-STA equally; dots show individual benchmarks.}
    \label{fig:onset_tracking}
\end{figure}

As the target occurs later, \method{} shifts its first temporal mention
accordingly, while TaRO remains concentrated near the video start.
Naive IoU follows the same overall trend as \method{} but is biased toward
earlier positions. All three methods attain high target-hit rates when the
event begins at the video start; for later events, \method{} largely preserves
its hit rate while the comparison policies decline.
Figure~\ref{fig:qualitative-main} illustrates this difference in two selected examples.

\begin{figure}[h]
    \centering
    \captionsetup[subfigure]{justification=centering,singlelinecheck=true}
    \begin{subfigure}[b]{0.49\linewidth}
        \centering
        \includegraphics[width=\linewidth]{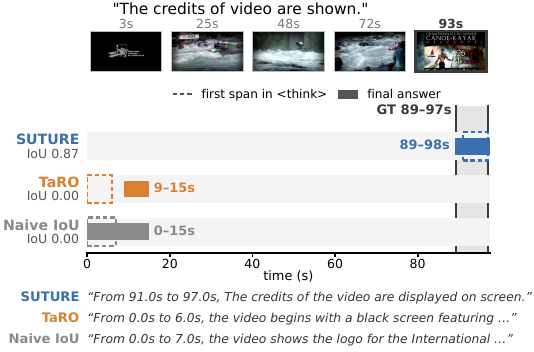}
        \caption{Video credits.}
    \end{subfigure}\hfill
    \begin{subfigure}[b]{0.49\linewidth}
        \centering
        \includegraphics[width=\linewidth]{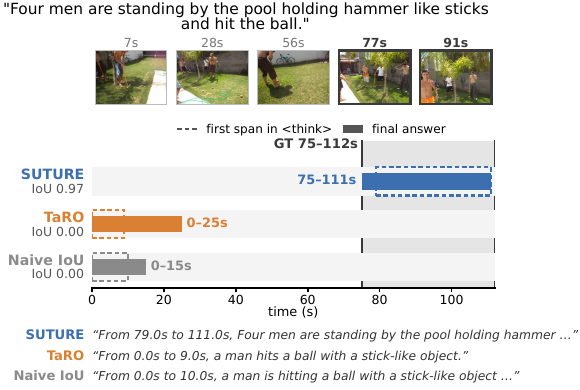}
        \caption{Activity by the pool.}
    \end{subfigure}
    \caption{\textbf{Later events with similar opening scenes.}
    Dashed outlines mark first temporal mentions and solid bars mark final
    answers.
    These selected ActivityNet examples and the selection rule are detailed in
    Appendix~\ref{app:qualitative_activitynet}.}
    \label{fig:qualitative-main}
\end{figure}

\begin{wraptable}{r}{0.5\textwidth}
    \captionsetup{font=small,skip=4pt}
    \caption{\textbf{Mean reasoning-trace length} in tokens.}
    \label{tab:reasoning_length_main}
    \centering
    \small
    \setlength{\tabcolsep}{3pt}
    \resizebox{0.9\linewidth}{!}{
    \begin{tabular}{lrrr}
        \toprule
        Benchmark & Naive IoU & TaRO & \method{} \\
        \midrule
        Charades-STA & 170.7 & 124.7 & \textbf{77.6} \\
        QVHighlights & 423.2 & 167.2 & \textbf{119.9} \\
        ActivityNet & 420.8 & 170.3 & \textbf{123.6} \\
        \bottomrule
    \end{tabular}}
\end{wraptable}

\paragraph{Generated response length.}
We also compare output length with TaRO using the
Qwen2.5-VL-7B-Instruct evaluation outputs for
Table~\ref{tab:main_results}.
\method{} reduces mean output length by
$31.6\%$ on Charades-STA, $24.2\%$ on
QVHighlights, and $23.5\%$ on ActivityNet.

The reduction is concentrated in the reasoning trace
(Table~\ref{tab:reasoning_length_main}); mean final answers are about one
token longer than TaRO's. The reward has no explicit length penalty.
These are generated-token counts, not latency measurements.
Appendix~\ref{app:response_length} reports the full single-run statistics.
\par\WFclear

\subsection{Local geometry of the structured residual}
\label{sec:update_geometry}

Reward ordering alone does not determine how response likelihoods
change under a parameter update.
We therefore measure each response's directional derivative of
log-likelihood normalized by response length along the diagnostic IoU and
residual directions.
Support alignment measures whether responses covering relatively
less supported target segments have larger local likelihood responses. The
residual-sign preference gap compares local likelihood responses
between rollouts with positive and negative centered reward corrections.

We analyze $35$ selected groups from severe Naive IoU failures:
$16$ recovered and $19$ unrecovered, classified retrospectively
by final \method{} outcomes. Reporting both tests whether the
preference extends to unrecovered groups; it does not establish
a cause of recovery.

We hold rollouts from the first visit and model parameters fixed. Both directions
use score vectors from the language model head and the last two decoder layers,
normalized by response length. This local diagnostic does not
reconstruct the complete training update. Appendix~\ref{app:update_geometry_metrics} details
group selection, diagnostic definitions, and uncertainty estimation.

\begin{figure}[h]
    \centering
    \begin{subfigure}[t]{0.38\linewidth}
        \centering
        \includegraphics[width=\linewidth]
            {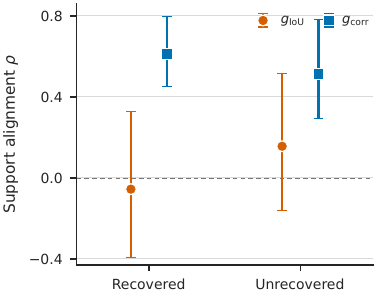}
        \caption{Temporal support alignment.}
        \label{fig:prop2-support}
    \end{subfigure}\hspace{0.06\linewidth}
    \begin{subfigure}[t]{0.38\linewidth}
        \centering
        \includegraphics[width=\linewidth]
            {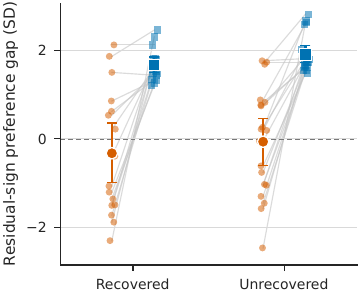}
        \caption{Residual-sign preference gap.}
        \label{fig:prop2-preference}
    \end{subfigure}
    \caption{\textbf{Residual gradient diagnostic.}
    \textbf{(a)} Pooled Spearman correlation between temporal support
    deficits and the mean directional responses of rollouts covering
    each segment.
    \textbf{(b)} Mean within-group standardized response difference
    between rollouts with positive and negative centered reward corrections.
    Error bars denote 95\% bootstrap intervals with rollout groups as clusters;
    definitions are in Appendix~\ref{app:update_geometry_metrics}.}
    \label{fig:prop2_semantics}
\end{figure}

\paragraph{Local response preference.}
The residual direction shows positive support alignment in
recovered ($\rho=0.614$) and unrecovered ($\rho=0.513$) groups,
compared with $-0.056$ and $0.156$ for the diagnostic IoU direction
(Figure~\ref{fig:prop2-support}).
Its mean standardized residual-sign preference gaps are $1.664$
and $1.903$
(Figure~\ref{fig:prop2-preference}).
For both metrics, the 95\% bootstrap intervals for the contrast between the residual
and IoU directions lie above zero in each case, using rollout groups as clusters.
Thus, the response preference associated with group support appears in both
recovered and unrecovered groups.
\endgroup

\section{Conclusion}

We study temporal RLVR through verifier design. \method{} uses coverage across
the realized rollout group to reweight overlap within the annotated target,
while disagreement sets the reweighting amplitude above a floor. The annotation remains the
sole source of correctness. The resulting reward decomposes into IoU and a
group-conditioned covariance correction, which favors less-supported correct
overlap when IoU statistics are matched.

Under shared training settings, fixed structured weighting improves over
Naive IoU, and the adaptive rule outperforms the tested constant amplitudes.
A local gradient diagnostic finds a support-dependent preference in the
selected groups, including those that do not recover. For later events,
the first temporal mention in \method{}-trained reasoning traces is less
often anchored to the video start. Grounding improves at every reported IoU threshold
across five benchmarks, with gains on the three primary benchmarks also
observed with a second MLLM backbone. These results show that the rollout
group can inform verification as well as reward normalization.

\newpage
\section*{Generative AI Use Disclosure}

Generative AI tools were used solely to improve the readability, grammar, and wording of the manuscript. They were not used to develop the research methodology, formulate hypotheses or theoretical claims, design experiments, interpret results, or generate scientific conclusions. All revisions were reviewed and approved by the authors, who take full responsibility for the final content of the paper.
\bibliography{references}
\bibliographystyle{iclr2027_conference}

\clearpage
\appendix
\pdfbookmark[0]{Appendix}{appendix.overview}
\setcounter{figure}{0}
\renewcommand{\thefigure}{\Alph{figure}}
\renewcommand{\theHfigure}{appendix.\arabic{figure}}
\setcounter{table}{0}
\renewcommand{\thetable}{\Alph{table}}
\renewcommand{\theHtable}{appendix.\arabic{table}} %
\begingroup
\raggedbottom
\captionsetup{font=small,labelfont=bf,skip=6pt}
\captionsetup[subfigure]{font=small,labelfont=normalfont,
  justification=raggedright,singlelinecheck=false,skip=4pt}
\renewcommand{\arraystretch}{1.1}
\setlength{\textfloatsep}{12pt plus 2pt minus 2pt}
\setlength{\intextsep}{10pt plus 2pt minus 2pt}
\let\appendixsection\section
\renewcommand{\section}{\needspace{7\baselineskip}\appendixsection}
\let\appendixsubsection\subsection
\renewcommand{\subsection}{\needspace{5\baselineskip}\appendixsubsection}

\section{Related Work}
\label{sec:related_work}

\paragraph{MLLMs for video temporal grounding.}
Standard temporal grounding benchmarks include QVHighlights
\citep{lei2021detecting}, ActivityNet Captions
\citep{krishna2017dense}, and Charades-STA
\citep{gao2017tall}. Recent work has increasingly
formulated video temporal grounding as direct generation with multimodal large
language models. ChatVTG
\citep{Qu_2024_CVPR} performs zero-shot grounding by matching language queries
against segment descriptions at multiple granularities, while TimeChat
\citep{ren2024timechattimesensitivemultimodallarge} introduces timestamp-aware representations for
understanding of long videos. HawkEye \citep{wang2024hawkeye} trains a text-to-text
video LLM on captions for video segments and negative spans from InternVid-G, using
coarse labels for temporal regions followed by recursive grounding to progressively
narrow the predicted interval. VTimeLLM \citep{huang2024vtimellm} develops
boundary-aware training for fine-grained temporal reasoning, and TimeSuite
\citep{zeng2025timesuite} incorporates explicit grounding supervision into
instruction tuning on long videos. VideoChat-Flash \citep{li2026videochatflash}
uses hierarchical compression of video tokens and training from short to long videos for
efficient long-context video modeling. More recent approaches further model
temporal structure explicitly. TRACE \citep{guo2025trace} represents videos as
causal event sequences, while VideoChat-R1.5 \citep{yan2025videochatr} uses
visual test-time scaling to iteratively refine attention over high-confidence
spatiotemporal regions. Foresee-to-Ground
\citep{zheng2026foreseetoground} restructures prediction as identification followed by measurement
using explicit temporal evidence, while One-to-Many Temporal Grounding
\citep{xu2026onetomany} extends the task from a single interval to multiple
disjoint occurrences of an event. These works primarily improve temporal
representation, supervision, task formulation, or model architecture; our
focus is instead on how temporal predictions are verified during adaptation with reinforcement learning.

\paragraph{Reinforcement learning for temporal grounding.}
Time-R1 \citep{wang2025timer} introduces RLVR for temporal video grounding
and optimizes generated intervals using a timestamp-aware IoU reward,
\begin{equation}
\label{eq:rw_timer1}
R_i^{\mathrm{tIoU}}
=
\operatorname{IoU}(P_i,\mathcal T)
\left(1-\frac{|\hat s_i-s^\star|}{D_V}\right)
\left(1-\frac{|\hat e_i-e^\star|}{D_V}\right),
\end{equation}
combined with a binary format reward as
$r_i^{\mathrm{Time\text{-}R1}}=R_i^{\mathrm{format}}+R_i^{\mathrm{tIoU}}$.
Its TimeRFT curriculum further selects moderately difficult examples and
progressively focuses optimization on harder grounding instances.

TaRO \citep{zheng2026temporalaware} shifts attention toward the quality of the
intermediate temporal reasoning. For rollout $i$, it measures the decrease in
reasoning log-likelihood after perturbing frames around the ground truth interval
boundaries,
\begin{equation}
    d_i = p_i-q_i,
    \qquad
    R_i^{\mathrm{temp}}
    =
    \alpha\,\ind[d_i>\bar d],
\end{equation}
and applies this term only when the predicted interval exceeds an IoU
threshold,
\begin{equation}
\label{eq:rw_taro}
r_i^{\mathrm{TaRO}}
=
R_i^{\mathrm{format}}
+
R_i^{\mathrm{tIoU}}
+
R_i^{\mathrm{temp}}
\ind[\operatorname{IoU}_i>\tau].
\end{equation}
TaRO therefore introduces group-relative supervision for temporal reasoning,
whereas its grounding accuracy term continues to score the predicted interval
against the ground truth interval. Our focus is different: \method{} conditions the
\emph{temporal overlap verifier itself} on the joint geometry of the realized
rollout group, using the group to determine how reward is distributed within
the target.

\section{Optimization Preliminaries}
\label{app:full_grpo}

For completeness, we give the full clipped GRPO objective underlying the
compact formulation in Section~\ref{sec:prelim_grpo}. Given the advantages $A_i$
normalized within each group in Equation~\ref{eq:grpo_advantage}, the
policy ratio for token $t$ of rollout $y_i$ is
\begin{equation}
    \rho_{i,t}(\theta)
    =
    \frac{
        \pi_{\theta}(y_{i,t}\mid x,y_{i,<t})
    }{
        \pi_{\theta_{\mathrm{old}}}(y_{i,t}\mid x,y_{i,<t})
    }.
    \label{eq:grpo_ratio}
\end{equation}
For rewards assigned to complete responses, the same advantage $A_i$ is applied to all tokens
in rollout $i$. Following the clipped policy ratio construction of PPO
\citep{schulman2017ppo}, the GRPO surrogate is
\begin{align}
\mathcal{J}_{\mathrm{GRPO}}(\theta)
=
\frac{1}{G}\sum_{i=1}^{G}
\frac{1}{|y_i|}
\sum_{t=1}^{|y_i|}
\Big[
&
\min\big(
    \rho_{i,t}(\theta) A_i,\,
    \operatorname{clip}(
        \rho_{i,t}(\theta),
        1-\epsilon,
        1+\epsilon
    )A_i
\big)
\nonumber\\
&-
\lambda_{\mathrm{KL}}
D_{\mathrm{KL},i,t}(\theta)
\Big],
\label{eq:grpo_objective}
\end{align}
where $\epsilon$ is the clipping radius, $\lambda_{\mathrm{KL}}$ controls KL
regularization, and $D_{\mathrm{KL},i,t}$ denotes the KL penalty for each token
with respect to the reference policy.

At the behavior policy $\theta=\theta_{\mathrm{old}}$,
$\rho_{i,t}(\theta_{\mathrm{old}})=1$. When clipping is locally inactive, the part of
the realized update determined by the reward reduces to the rollout scores
$\mathbf z_i$ and direction $\mathbf d(r)$ defined in
Equations~\ref{eq:rollout_score} and~\ref{eq:local_reward_update}. The KL term
remains a separate objective component and is unchanged by the verifier.

\section{Extended Theoretical Characterization}
\label{app:theory_characterization}

This section collects secondary identities and formal consequences supporting
the characterization in the main text. We characterize the magnitude of the
variation of the normalized weight field, state the induced ordering by group
support formally, and record the normalization identity underlying the update analysis.

\subsection{Refinement of the normalized field and structured variation}
The parallel definitions in Equations~\ref{eq:iou} and~\ref{eq:suture_reward}
give, for member $P_i$,
\begin{equation}
\label{eq:integrand_correction}
R_i^{\mathrm S}-R_i^{\mathrm{IoU}}
=
\frac{\displaystyle\int_{P_i\cap\mathcal T}
\bigl(\widetilde w_{\mathcal P}(t)-1\bigr)\,dt}
{L+F_i}.
\end{equation}
Proposition~\ref{prop:refinement} further shows that the normalized field is a
centered perturbation of uniform IoU weighting through
Equation~\ref{eq:pointwise_refinement}.

We measure the RMS magnitude of this structured field relative to uniform IoU
weighting by
\begin{equation}
\label{eq:scorr}
S_{\mathrm{corr}}(\mathcal P)
:=
\left[
\frac{1}{L}\int_{\mathcal T}
\bigl(\widetilde w_{\mathcal P}(t)-1\bigr)^2dt
\right]^{1/2}.
\end{equation}
Because $L^{-1}\int_{\mathcal T}\widetilde w_{\mathcal P}(t)dt=1$,
Equation~\ref{eq:pointwise_refinement} gives
\begin{equation}
\label{eq:scorr_factorization}
S_{\mathrm{corr}}(\mathcal P)
=
\beta_{\mathcal P}
\operatorname{Std}\!\left[\phi_{\mathcal P}(U)\right],
\qquad U\sim\mathrm{Unif}(\mathcal T).
\end{equation}
Thus, if the realized residual is constant almost everywhere within the target,
the normalized field is uniform and \method{} recovers ordinary IoU for every
prediction. More generally, $S_{\mathrm{corr}}$ measures the structured
variation available in the verifier field; the correction received by member
$P_i$ additionally depends on its alignment with that field through $\xi_i$.
Section~\ref{app:structured_activation} examines this variation in realized
training groups.

\subsection{Ordering by group support}
For a scored member $P_i$ with $I_i>0$, define the average realized group
coverage over its correctly covered target region as
\begin{equation}
\label{eq:avg_scored_coverage}
\bar C_i
:=
\frac{1}{I_i}
\int_{P_i\cap\mathcal T}C_{\mathcal P}(t)\,dt.
\end{equation}

\begin{corollary}[Ordering by group support]
\label{cor:support_order}
For a realized group with $a_{\mathcal P}>0$, consider two scored members
$P_i$ and $P_j$ with $I_i=I_j>0$ and $F_i=F_j$. If
$\bar C_i<\bar C_j$, then $R_i^{\mathrm S}>R_j^{\mathrm S}$.
\end{corollary}

This formalizes the preference of the structured overlap component illustrated in
Figure~\ref{fig:main}. Under matched aggregate overlap statistics, the local
support field determines the ordering: correctly covering a relatively
less supported portion of the target receives the larger structured
score, while $\beta_{\mathcal P}$ controls the size of this reward difference. The
proof is given in Appendix~\ref{app:proofs}.

\subsection{Group-relative normalization identity}
For the controlled plain IoU comparison, let
$r_i^{\mathrm S}=r_i^{\mathrm{IoU}}+\chi_i$ and
$\widetilde\chi_i=\chi_i-\bar\chi$, with $\kappa_{\mathcal P}$ and
$\alpha_{\mathcal P}$ defined in Section~\ref{sec:theory}. Let
$A_i^{\mathrm{IoU}}$ and $A_i^{\mathrm S}$ denote the corresponding
advantages normalized within each group. Centering the two implemented rewards gives
the exact identity
\begin{equation}
\label{eq:advantage_decomp}
A_i^{\mathrm S}
=
\kappa_{\mathcal P}A_i^{\mathrm{IoU}}
+
\alpha_{\mathcal P}\widetilde\chi_i.
\end{equation}
Multiplying Equation~\ref{eq:advantage_decomp} by $\mathbf z_i$ and averaging
over the fixed rollout group yields Equation~\ref{eq:gradient_decomp_body}.
The identity permits degenerate cases: a correction shared by the whole group centers to
zero, a nonconstant correction may cancel in parameter space, and a nonzero
$\mathbf d_{\mathrm{corr}}$ may still be collinear with the IoU direction.
Section~\ref{sec:update_geometry} examines the realized geometry empirically.

\section{Proofs and Basic Properties}
\label{app:proofs}

We establish positivity of the normalized field, prove the exact reward
refinement and ordering by group support, and derive the structured update
component used in the main text.

\subsection{Well-definedness and basic properties}
\label{app:well_defined}
\begin{lemma}[Positivity of the normalized field]
\label{lem:positivity}
On the parameter domain $G\ge2$ and $0\le u_{\mathrm{floor}}<1$,
$1+a_{\mathcal P}\phi_{\mathcal P}(t)>0$ on every open target segment of
positive length between consecutive interval boundaries. Hence the denominator
in Equation~\ref{eq:niw_weight} is positive and
$\widetilde w_{\mathcal P}(t)$ is well defined.
\end{lemma}

\begin{proof}
Partition $\mathcal T$ at the rollout interval boundaries and consider an
open segment $A\subseteq\mathcal T$ of positive length between consecutive
boundaries. Let $c$ be its constant group coverage. If $c=0$, then $\phi_{\mathcal P}=0$ and $1+a_{\mathcal P}\phi_{\mathcal P}=1$. If $0<c<1$, then $\phi_{\mathcal P}=1-2c>-1$ and $a_{\mathcal P}\le1$, so $1+a_{\mathcal P}\phi_{\mathcal P}>0$. If $c=1$, every rollout contains $A$; invalid parses therefore cannot occur in this case, and every pair has positive interval IoU. Hence $\mathrm{AvgPairIoU}(\mathcal P)>0$. Since $u_{\mathrm{floor}}<1$,
\begin{equation}
a_{\mathcal P}
=\max\bigl(u_{\mathcal P},u_{\mathrm{floor}}\bigr)<1,
\end{equation}
and $1+a_{\mathcal P}\phi_{\mathcal P}=1-a_{\mathcal P}>0$. These segments cover the target except for their boundary points, proving that the normalizing denominator is positive.

The segment boundaries form a finite set and therefore do not affect any integral.
Segments of zero length are discarded, and a prediction whose clipped endpoints coincide
contributes no coverage of positive length.
\end{proof}

Equation~\ref{eq:niw_weight} fixes $\int_{\mathcal T}\widetilde w_{\mathcal P}=L$. Positivity then gives $0\le R_i^{\mathrm S}\le L/(L+F_i)\le1$. If $a_{\mathcal P}=0$, then $\widetilde w_{\mathcal P}=1$ and \method{} recovers ordinary IoU. If every rollout misses $\mathcal T$, then $C_{\mathcal P}=\phi_{\mathcal P}=0$ on the target, so the field is uniform and both rewards give zero overlap.

\subsection{Exact refinement and preference based on group support}
An overlap verifier that depends only on $(I_i,F_i,L)$ assigns equal scores whenever two
predictions have the same aggregate overlap statistics, irrespective of where their
correct overlap lies. More generally, an independently computed reward for rollout $i$
is functionally independent of every other $P_m$, $m\ne i$; changing another rollout
cannot make that verifier adapt either its overall intervention strength or its
preference across positions. These are the two scales of information supplied by
conditioning \method{} on $\mathcal P$.

\paragraph{Proof of Proposition~\ref{prop:refinement}.}
For member $P_i$, define
\begin{equation}
\label{eq:BM_def}
B_{\mathcal P}=\int_{\mathcal T}\phi_{\mathcal P}(t)dt,
\qquad
M_i=\int_{\mathcal T}\phi_{\mathcal P}(t)\ind[t\in P_i]dt.
\end{equation}
Since $B_{\mathcal P}=L\bar\phi_{\mathcal P}$, Equation~\ref{eq:niw_weight}
gives
\begin{align}
\widetilde w_{\mathcal P}(t)-1
&=
\frac{1+a_{\mathcal P}\phi_{\mathcal P}(t)}
{1+a_{\mathcal P}\bar\phi_{\mathcal P}}-1 \notag\\
&=
\frac{a_{\mathcal P}}
{1+a_{\mathcal P}\bar\phi_{\mathcal P}}
\bigl(\phi_{\mathcal P}(t)-\bar\phi_{\mathcal P}\bigr),
\end{align}
which is Equation~\ref{eq:pointwise_refinement}. The raw target mass is
$L+a_{\mathcal P}B_{\mathcal P}$, while the raw mass covered by $P_i$ is
$I_i+a_{\mathcal P}M_i$. Substitution into Equation~\ref{eq:suture_reward}
gives
\begin{align}
R_i^{\mathrm S}-R_i^{\mathrm{IoU}}
&=\frac{L(I_i+a_{\mathcal P}M_i)-I_i(L+a_{\mathcal P}B_{\mathcal P})}
{(L+a_{\mathcal P}B_{\mathcal P})(L+F_i)} \notag\\
&=\frac{a_{\mathcal P}}{1+a_{\mathcal P}B_{\mathcal P}/L}
\frac{M_i-I_iB_{\mathcal P}/L}{L+F_i}.
\label{eq:refinement_proof}
\end{align}
For $U\sim\mathrm{Unif}(\mathcal T)$,
\begin{equation}
M_i-\frac{I_i}{L}B_{\mathcal P}
=L\,\Cov\bigl(\phi_{\mathcal P}(U),\ind[U\in P_i]\bigr).
\end{equation}
Together with $\bar\phi_{\mathcal P}=B_{\mathcal P}/L$, this is exactly
Equation~\ref{eq:main_refinement}. Lemma~\ref{lem:positivity} makes its
denominator positive; thus $\beta_{\mathcal P}\ge0$, with strict inequality
whenever $a_{\mathcal P}>0$.

\paragraph{Proof of Corollary~\ref{cor:support_order}.}
Each scored member $P_i$ contributes to its own coverage, so $C_{\mathcal P}(t)\ge1/G>0$
on $P_i\cap\mathcal T$. Using the shorthand from Equation~\ref{eq:avg_scored_coverage},
the branch for zero support in Equation~\ref{eq:residual} never enters this integral, so
\begin{equation}
\label{eq:scored_member}
M_i
=I_i\bigl(1-2\bar C_i\bigr).
\end{equation}
For members $P_i,P_j$ with common $I>0$ and $F$, the IoU terms and the baseline over the
target in Equation~\ref{eq:refinement_proof} cancel. Applying
Equation~\ref{eq:scored_member} yields
\begin{equation}
R_i^{\mathrm S}-R_j^{\mathrm S}
=\frac{2\beta_{\mathcal P}I}{L+F}
\bigl(\bar C_j-\bar C_i\bigr),
\end{equation}
which proves the claim. This ordering concerns the structured overlap channel alone; unchanged format terms may still separate two responses.

The convention for zero coverage concerns normalization of target mass
rather than the ordering among covered locations.
For a fixed rollout group and amplitude, applying
$1-2C_{\mathcal P}(t)$ also at uncovered locations would leave
every member's raw overlap numerator unchanged, while increasing
the common normalization denominator whenever $a_{\mathcal P}>0$
and the uncovered target region has positive duration.
The resulting temporal rewards would therefore differ by a
common positive scale, preserving their ordering within the group.
This does not imply identical training updates when format
rewards and advantage normalization are included.

A zero residual specifies the raw baseline, not a normalized
weight of one.
By Equation~\ref{eq:pointwise_refinement}, the normalized weight
exceeds one precisely when
$\phi_{\mathcal P}(t)>\bar\phi_{\mathcal P}$, provided
$a_{\mathcal P}>0$.
Thus, half coverage is the neutral point of the raw contrast,
not a fixed threshold for normalized upweighting.

\subsection{Derivation of the structured update component}
\label{app:grad_decomp}
For the controlled verifier comparison,
\begin{equation}
r_i^{\mathrm S}=R_i^{\mathrm{format}}+R_i^{\mathrm S},
\qquad
r_i^{\mathrm{IoU}}=R_i^{\mathrm{format}}+R_i^{\mathrm{IoU}}.
\end{equation}
Define
\begin{equation}
\chi_i=r_i^{\mathrm S}-r_i^{\mathrm{IoU}}
=\beta_{\mathcal P}\xi_i,
\qquad
\widetilde\chi_i=\chi_i-\bar\chi.
\end{equation}
Centering $r_i^{\mathrm S}=r_i^{\mathrm{IoU}}+\chi_i$ and dividing by
$\sigma_{\mathrm S}+\delta$ gives Equation~\ref{eq:advantage_decomp}.
Multiplying that identity by $\mathbf z_i$ and averaging over the fixed rollout group
yields Equation~\ref{eq:gradient_decomp_body}.

\section{Training and Evaluation Details}
\label{app:training_details}

\subsection{Training data and optimization}

The settings in this section apply to TaRO and the controlled verifier
variants. Time-R1 is reproduced separately following its original
training settings \citep{wang2025timer}.
We train all controlled variants on 2.5K examples from the
timerft\_data split of TimeR1-Dataset \citep{wang2025timer}.
We follow the training framework and optimization setup of TaRO
\citep{zheng2026temporalaware}. For the primary Qwen2.5-VL-7B-Instruct experiments,
training runs for three epochs (468 optimization steps) on four NVIDIA B200
GPUs. Each step uses a prompt batch size of 16 and $G=8$ rollouts per prompt,
for 128 sampled responses before minibatch updates. We use a learning rate of
$10^{-6}$, sampling temperature $1.0$, and a KL loss coefficient of $0.01$;
the KL term is applied in the optimization loss rather than added to the
reward. The maximum prompt and response lengths are 7168 and 1024 tokens,
respectively. The primary \method{} configuration uses
$u_{\mathrm{floor}}=0.3$.

\subsection{Warmup using captions}
As in the TaRO training framework, the first epoch (156 steps) uses a reasoning warmup
using captions. When temporal captions are available, a nonempty random subset of
caption segments is prepended to the model response
as a partial reasoning prefix before the remaining tokens are generated. The
injected prefix is included in the response tokens used for optimization.
Because these prefixed trajectories are not sampled entirely from the current
policy, they are marked as off-policy; under the framework's masking rule,
updates for such trajectories are suppressed when the advantage is negative and retained
when it is positive. The warmup is disabled after the
first epoch, and the remaining two epochs use on-policy GRPO rollouts.
All controlled reward variants use the same warmup procedure. Consequently,
the timestamped reasoning format induced by this warmup is shared across
reward variants; our analyses of reasoning traces compare how the trained policies
use that common format rather than attributing the format itself to \method{}.

\subsection{Evaluation and video preprocessing}
We evaluate on the QVHighlights validation split (1,550 examples),
Charades-STA test split (3,720), and ActivityNet Captions val1 split(15,933).
For each backbone and benchmark, TaRO and the controlled verifier
variants use the same split, prompt, video preprocessing, and decoding
settings.
Table~\ref{tab:train_eval_settings} lists the training and evaluation
settings for the primary Qwen2.5-VL-7B-Instruct experiments.

\begin{table}[H]
\caption{\textbf{Training and evaluation settings.}
TaRO and all controlled reward variants share these settings within
each stage for Qwen2.5-VL-7B-Instruct. For visual patch size $p$, the
visual token budget $B$ sets total pixels to $B(2p)^2$.}
\label{tab:train_eval_settings}
\centering
\small
\renewcommand{\arraystretch}{1.0}
\begin{tabularx}{\linewidth}{@{}lXX@{}}
\toprule
Setting & Training rollouts & Evaluation \\
\midrule
Decoding & Sampling & Greedy \\
Temperature & 1.0 & 0.0 \\
Top-$p$ / top-$k$ & $1.0$ / $-1$ & $1.0$ / $-1$ \\
Responses per prompt & 8 & 1 \\
Response limit (tokens) & 1,024 & 1,024 \\
Visual token budget $B$ & 3,584 & 4,096 \\
Maximum frames & 768 & 1,024 \\
Requested frame rate (fps) & 2.0 (library default) & 2 (explicit) \\
Minimum frame pixels & $4(2p)^2$ & $4(2p)^2$ \\
Prompt template & TaRO \texttt{r1}, single interval & TaRO \texttt{r1}, single interval \\
Prompt limit (tokens) & 7,168 & --- \\
Model context limit (tokens) & 8,192 & 16,384 \\
\bottomrule
\end{tabularx}
\end{table}

\paragraph{Video preprocessing.}
Training uses the default 2.0 fps setting of qwen-vl,
while evaluation explicitly sets 2 fps.
For Ego4D-NLQ and TACoS, the evaluation video token budget is increased
to 8,192 while retaining the same requested frame rate across reward
variants. Qwen3-VL-8B-Instruct uses a visual patch size of 16 pixels.

\section{Full Benchmark Results}
\label{app:full_benchmark_results}

For completeness, Table~\ref{tab:full_benchmark_results} reports the broader
comparison omitted from the condensed main table. Comparator rows other than
Time-R1 and TaRO are transcribed from Table~1 of TaRO
\citep{zheng2026temporalaware}; the original methods are described and cited in
Section~\ref{sec:related_work}. We reproduce Time-R1 separately under its
original training settings \citep{wang2025timer} and reevaluate TaRO on
all three datasets for a direct comparison. We omit
VideoChat-R1.5 \citep{yan2025videochatr} results on Charades-STA because Charades data were included in
its training, making a direct comparison on that benchmark unfair. All values
in that table are R1 at the stated IoU threshold, shown as percentages
rounded to one decimal place.

\begin{table}[H]
\caption{\textbf{Full video temporal grounding comparison.} The table includes
the broader set of temporal grounding baselines alongside the direct TaRO and
\method{} comparisons retained in Table~\ref{tab:main_results}.}
\label{tab:full_benchmark_results}
\centering
\small
\resizebox{\linewidth}{!}{%
\begin{tabular}{lcccccccccc}
\toprule
& & \multicolumn{3}{c}{QVHighlights} & \multicolumn{3}{c}{ActivityNet} & \multicolumn{3}{c}{Charades-STA} \\
\cmidrule(lr){3-5}\cmidrule(lr){6-8}\cmidrule(lr){9-11}
Method & Size & R1@0.3 & R1@0.5 & R1@0.7 & R1@0.3 & R1@0.5 & R1@0.7 & R1@0.3 & R1@0.5 & R1@0.7 \\
\midrule
ChatVTG          & 7B & --    & --    & --   & 40.7 & 22.5 & 9.4  & 52.7 & 33.0 & 15.9 \\
TimeChat         & 7B & --    & 8.3   & 4.3  & 36.2 & 20.2 & 9.5  & --   & 32.2 & 13.4 \\
HawkEye          & 7B & --    & --    & --   & 49.1 & 29.3 & 10.7 & 50.6 & 31.4 & 14.5 \\
VTimeLLM         & 7B & --    & 26.1  & 11.1 & 44.0 & 27.8 & 14.3 & 51.0 & 27.5 & 11.4 \\
TimeSuite        & 7B & --    & 12.3  & 9.2  & --   & 16.6 & 9.3  & 69.9 & 48.7 & 24.0 \\
VideoChat-Flash  & 7B & --    & --    & --   & --   & --   & --   & 74.5 & 53.1 & 27.6 \\
TRACE            & 7B & --    & --    & --   & --   & --   & --   & --   & 40.3 & 19.4 \\
\midrule
\multicolumn{11}{l}{\textit{Qwen2.5-VL-7B-Instruct as base model}} \\
Qwen2.5-VL-7B-Instruct & 7B & 15.9 & 7.1  & 4.2  & 24.4 & 13.6 & 6.7  & 72.5 & 53.6 & 28.5 \\
VideoChat-R1.5   & 7B & 71.4  & 55.8 & 38.4 & 52.4 & 32.3 & 16.8 & --   & --   & --   \\
Time-R1          & 7B & 80.3  & 66.2 & 44.8 & 58.6 & 39.0 & 21.4 & 78.1 & 60.8 & 35.3 \\
TaRO             & 7B & 81.7  & 66.8 & 45.3 & 61.0 & 39.7 & 20.5 & 79.3 & 63.3 & 36.7 \\
\method{} (ours) & 7B & \textbf{84.1} & \textbf{70.1} & \textbf{50.6} & \textbf{66.8} & \textbf{47.1} & \textbf{26.2} & \textbf{81.6} & \textbf{65.1} & \textbf{37.6} \\
\midrule
\multicolumn{11}{l}{\textit{Qwen3-VL-8B-Instruct as base model}} \\
TaRO             & 8B & 83.2 & 68.7 & 51.7 & 58.0 & 39.8 & 23.9 & 82.9 & 67.0 & 37.7 \\
\method{} (ours) & 8B & \textbf{84.3} & \textbf{70.0} & \textbf{54.4} & \textbf{63.0} & \textbf{44.5} & \textbf{27.3} & \textbf{83.6} & \textbf{67.4} & \textbf{39.9} \\
\bottomrule
\end{tabular}}
\end{table}

\subsection{Results across training runs}
\label{app:training_runs}

We compare TaRO and \method{} across three training runs.
Table~\ref{tab:training_runs} reports the mean and sample standard
deviation across runs. \method{} achieves higher mean scores on
all four metrics for each benchmark.

\begin{table}[H]
\caption{\textbf{Results across three training runs.}
Values are percentages, reported as mean $\pm$ sample standard
deviation and rounded to one decimal place.}
\label{tab:training_runs}
\centering
\small
\begin{tabular}{llcccc}
\toprule
Benchmark & Method & mIoU & R1@0.3 & R1@0.5 & R1@0.7 \\
\midrule
QVHighlights & TaRO & $56.5 \pm 4.6$ & $78.7 \pm 3.4$ & $63.2 \pm 5.7$ & $41.4 \pm 7.6$ \\
 & \method{} & $\mathbf{61.4} \pm 1.1$ & $\mathbf{83.0} \pm 1.0$ & $\mathbf{69.3} \pm 1.3$ & $\mathbf{48.3} \pm 2.0$ \\
\midrule
ActivityNet & TaRO & $39.4 \pm 1.9$ & $58.2 \pm 3.1$ & $38.2 \pm 2.1$ & $19.6 \pm 1.6$ \\
 & \method{} & $\mathbf{44.3} \pm 1.8$ & $\mathbf{64.8} \pm 2.3$ & $\mathbf{44.9} \pm 2.2$ & $\mathbf{24.4} \pm 1.5$ \\
\midrule
Charades-STA & TaRO & $54.2 \pm 0.9$ & $79.0 \pm 0.5$ & $62.8 \pm 1.3$ & $35.6 \pm 2.2$ \\
 & \method{} & $\mathbf{55.5} \pm 0.4$ & $\mathbf{80.8} \pm 0.7$ & $\mathbf{64.3} \pm 0.9$ & $\mathbf{37.1} \pm 0.4$ \\
\bottomrule
\end{tabular}
\end{table}

\section{Additional Experiments and Analyses}
\label{app:additional_experiments}

We first report sensitivity to group size, verifier ablations for each benchmark, and
sensitivity to the disagreement floor. We then examine the structured weight
field in realized rollout groups and compare verifier components by event start
time, followed by generated response length.

\subsection{Sensitivity to rollout group size}
\label{app:group_size}
Table~\ref{tab:group_size_sensitivity} varies $G$ while holding the remaining
configuration fixed. The primary $G=8$ configuration performs best overall,
with the clearest separation at stricter localization thresholds.

\begin{table}[H]
\caption{Sensitivity to rollout group size using Qwen2.5-VL-7B-Instruct. All numbers are
R1 at the stated IoU threshold and are reported as percentages rounded to one decimal
place. The $G=8$ results are reproduced from Table~\ref{tab:main_results}.}
\label{tab:group_size_sensitivity}
\centering
\small
\begin{tabular}{ccccccc}
\toprule
& \multicolumn{3}{c}{QVHighlights} & \multicolumn{3}{c}{Charades-STA} \\
\cmidrule(lr){2-4}\cmidrule(lr){5-7}
Group size $G$ & R1@0.3 & R1@0.5 & R1@0.7 & R1@0.3 & R1@0.5 & R1@0.7 \\
\midrule
2 & 82.2 & 66.8 & 44.7 & 79.3 & 58.3 & 32.8 \\
4 & 81.7 & 66.9 & 47.5 & 80.5 & 62.7 & 35.4 \\
8 & \textbf{84.1} & \textbf{70.1} & \textbf{50.6} & \textbf{81.6} & \textbf{65.1} & \textbf{37.6} \\
\bottomrule
\end{tabular}
\end{table}

\subsection{Verifier ablations by benchmark}
\label{app:constant_amplitude}
Table~\ref{tab:constant_amplitude_details} gives the results for each benchmark
behind the rows for Naive IoU and fixed amplitudes in
Table~\ref{tab:amplitude_ablation}. Naive IoU uses uniform target weighting;
the fixed controls retain the support construction at each position. Every fixed control
uses the same rule for constructing the support field at each position from the rollout group; fixing the
amplitude therefore does not make the verifier independent across rollouts,
nor does it match the realized reward or gradient distribution.

The $a=0.8$ control also appears in the main comparison in
Table~\ref{tab:amplitude_ablation}. Its QVHighlights and Charades-STA mIoUs are
$59.6$ and $55.4$, respectively; their average with equal weight is $57.5$ after
rounding, as reported in Table~\ref{tab:amplitude_ablation}.
Table~\ref{tab:constant_amplitude_details} additionally reports the $a=0.5$
control. All values are point estimates from a single run rounded to one decimal
place, without uncertainty intervals.

\begin{table}[H]
\caption{Results on QVHighlights and Charades-STA for the
Naive IoU reference and training controls with constant amplitudes. All values are
percentages rounded to one decimal place. The Naive IoU, $a=0.3$, $a=0.4$,
and $a=0.8$ settings correspond to the main comparison in
Table~\ref{tab:amplitude_ablation}; $a=0.5$ provides an additional
control with constant amplitude. The $a=0.4$ control is used in
Figure~\ref{fig:adaptive_mechanism}.}
\label{tab:constant_amplitude_details}
\centering
\small
\begin{tabular}{llcccc}
\toprule
Verifier & Benchmark & mIoU & R1@0.3 & R1@0.5 & R1@0.7 \\
\midrule
Naive IoU & QVHighlights & 57.8 & 79.7 & 64.6 & 43.5 \\
          & Charades-STA & 51.8 & 77.4 & 59.2 & 30.8 \\
\midrule
$a=0.3$ & QVHighlights & 58.5 & 79.7 & 65.2 & 44.3 \\
        & Charades-STA & 53.7 & 78.5 & 61.3 & 34.4 \\
\midrule
$a=0.4$ & QVHighlights & 59.7 & 81.6 & 67.2 & 46.4 \\
        & Charades-STA & 54.9 & 80.1 & 63.7 & 35.9 \\
\midrule
$a=0.5$ & QVHighlights & 60.8 & 83.2 & 67.7 & 47.1 \\
        & Charades-STA & 53.8 & 79.7 & 61.2 & 33.7 \\
\midrule
$a=0.8$ & QVHighlights & 59.6 & 81.9 & 66.8 & 45.6 \\
        & Charades-STA & 55.4 & 80.2 & 63.8 & 37.5 \\
\bottomrule
\end{tabular}
\end{table}

\subsection{Sensitivity to the disagreement floor}
\label{app:ufloor}

We examine the sensitivity of \method{} to the disagreement floor
$u_{\mathrm{floor}}$, which sets a lower bound on the redistribution amplitude for each
sample in Equation~\ref{eq:amplitude}. Table~\ref{tab:ufloor_sensitivity}
reports the simple average of QVHighlights and Charades-STA results for each
setting. This sweep varies the lower bound while preserving the amplitude rule based on
disagreement; the controls with constant amplitudes in
Table~\ref{tab:amplitude_ablation} instead test whether that rule can be
replaced by one fixed value.

Setting $u_{\mathrm{floor}}=0$ retains $a_{\mathcal P}=u_{\mathcal P}$ and
the group-conditioned support field, but removes the lower bound; it is not
the uniform IoU control. Fixing $a=0.3$ instead removes the amplitude's
dependence on disagreement. Neither setting matches the default adaptive
rule on the four benchmark-averaged metrics
(Tables~\ref{tab:amplitude_ablation} and~\ref{tab:ufloor_sensitivity}).

\begin{table}[H]
\caption{Sensitivity to the disagreement floor $u_{\mathrm{floor}}$. Results are the simple average of QVHighlights and Charades-STA. The default $u_{\mathrm{floor}}=0.3$ is highlighted.}
\label{tab:ufloor_sensitivity}
\centering
\small
\begin{tabular}{ccccc}
\toprule
$u_{\mathrm{floor}}$ & R1@0.3 & R1@0.5 & R1@0.7 & mIoU \\
\midrule
0.00 & 81.0 & 65.0 & 41.4 & 57.7 \\
0.15 & 81.0 & 64.8 & 40.7 & 57.1 \\
\textbf{0.30} & \textbf{82.9} & \textbf{67.6} & \textbf{44.1} & \textbf{59.2} \\
0.50 & 81.3 & 67.4 & 43.9 & 58.6 \\
0.70 & 80.4 & 65.2 & 41.6 & 57.1 \\
\bottomrule
\end{tabular}
\end{table}

A moderate disagreement floor performs best overall. Relative to removing the
floor ($u_{\mathrm{floor}}=0$), the default value of $0.3$ improves mean
R1@0.7 by $2.67$ points and mIoU by $1.54$ points, while $0.5$ retains
comparable performance. Increasing the floor further to $0.7$ removes most of
this gain. The sweep therefore supports a moderate nonzero floor rather than a
narrowly tuned optimum.

The corresponding activation statistics, measured late in training
(approximately after the second epoch), help interpret this nonmonotonic
trend. Across $u_{\mathrm{floor}}\in\{0,0.15,0.3,0.5,0.7\}$, the fraction of
realized groups satisfying $u_{\mathcal P}<u_{\mathrm{floor}}$ increases from
$0.0\%$ to $44.4\%$, $82.8\%$, $96.0\%$, and $99.4\%$, while
$\mathbb{E}[S_{\mathrm{corr}}]$ increases from $0.0663$ to $0.0740$, $0.1147$,
$0.2121$, and $0.3659$. The larger $S_{\mathrm{corr}}$ values indicate a less
uniform target-weight field, not necessarily more useful reward differences
between rollouts. Localization performance does not improve monotonically;
at a high floor, the amplitude also depends almost entirely on the floor
rather than realized disagreement. Appendix~\ref{app:structured_activation}
further compares target-weight variation across target durations.

\subsection{Structured variation across rollout groups}
\label{app:structured_activation}

Equation~\ref{eq:scorr_factorization} gives an exact factorization of
the variation over the target induced by structured verification:
\begin{equation}
\label{eq:scorr_empirical_factorization}
S_{\mathrm{corr}}(\mathcal P)
=
\beta_{\mathcal P}
\operatorname{Std}_{U\sim\mathrm{Unif}(\mathcal T)}
\left[\phi_{\mathcal P}(U)\right],
\end{equation}
where
\[
\beta_{\mathcal P}
=
\frac{a_{\mathcal P}}
{1+a_{\mathcal P}\bar\phi_{\mathcal P}},
\qquad
\bar\phi_{\mathcal P}
=
\frac{1}{L}
\int_{\mathcal T}\phi_{\mathcal P}(t)\,dt.
\]
Recall that
$\phi_{\mathcal P}(t)=0$ whenever $C_{\mathcal P}(t)=0$.
Thus, both the mean and standard deviation above are computed over
the full target, including locations with no rollout
coverage.

For the empirical analysis, we reconstruct the normalized target
weight field for each realized rollout group and measure
\begin{equation}
\label{eq:scorr_direct}
S_{\mathrm{corr}}^{\mathrm{direct}}(\mathcal P)
=
\left[
\frac{1}{L}
\int_{\mathcal T}
\left(
\widetilde w_{\mathcal P}(t)-1
\right)^2 dt
\right]^{1/2}.
\end{equation}

\paragraph{Variation across target duration.}
Table~\ref{tab:structured_activation_duration} stratifies realized
groups by target duration at the final analyzed stage.
Each column reports the median of the corresponding quantity across groups
within the duration bin.

The effective global coefficient $\beta_{\mathcal P}$ decreases
moderately with target duration, from $0.4264$ for 5--10s targets
to $0.3553$ for targets longer than 80s.
In contrast, the variation of the residual within the target changes much more
strongly: its median rises from $0.1000$ in the 5--10s bin to
$0.5000$ in the 40--80s bin.
The directly measured $S_{\mathrm{corr}}$ shows the corresponding
increase in structured field variation, from $0.0457$ for 5--10s
targets to $0.1750$ for 40--80s targets.

This separation is also visible when comparing the 5--10s and
40--80s bins:
$\beta_{\mathcal P}$ differs by only $1.15\times$, whereas
$\operatorname{Std}_{\mathcal T}(\phi_{\mathcal P})$
differs by $5.00\times$.
Thus, the duration dependence of the structured field is associated
primarily with changes in local residual variation rather than with
large changes in the global coefficient.

\begin{table}[H]
\caption{\textbf{Diagnostics of structured variation by target duration.}
$S_{\mathrm{corr}}$ is measured directly from the normalized target
weight field.
$\beta_{\mathcal P}$ and
$\operatorname{Std}_{\mathcal T}(\phi_{\mathcal P})$
are computed over the full target using the residual
definition in Equation~\ref{eq:residual}.
Entries are medians of the corresponding quantities across groups
within each duration bin.
Because the three columns are summarized independently, the product
of the two reported medians need not equal the reported median of
$S_{\mathrm{corr}}$.}
\label{tab:structured_activation_duration}
\centering
\small
\begin{tabular}{lcccc}
\toprule
Target duration
& $n$
& $S_{\mathrm{corr}}$
& $\beta_{\mathcal P}$
& $\operatorname{Std}_{\mathcal T}(\phi_{\mathcal P})$ \\
\midrule
5--10s
& 223
& 0.0457
& 0.4264
& 0.1000 \\

10--20s
& 241
& 0.1493
& 0.4022
& 0.3689 \\

20--40s
& 134
& 0.1696
& 0.3848
& 0.4403 \\

40--80s
& 67
& 0.1750
& 0.3723
& 0.5000 \\

80s+
& 16
& 0.1620
& 0.3553
& 0.4576 \\
\bottomrule
\end{tabular}
\end{table}

\subsection{Event-start-time analysis of verifier components}
\label{app:component_behavior}

The aggregate results do not show whether gains vary with event start time.
Figure~\ref{fig:adaptive_mechanism} compares Naive IoU, the mean-matched fixed
$a=0.4$ control, and \method{} on Charades-STA and QVHighlights.
We examine final-answer mIoU and the first well-formed temporal span in the
reasoning trace before the final answer (the first temporal mention).
On both benchmarks, Naive IoU and \method{} use step 350,
while the fixed control uses step 300.

Panel (a) retains examples with well-formed final answers from all three methods:
$3{,}718$ of $3{,}720$ on Charades-STA and $1{,}448$ of $1{,}550$ on QVHighlights.
The $102$ excluded QVHighlights responses lack the \texttt{<answer>} tag in
Naive IoU. Requiring valid first temporal mentions for all three methods
leaves the same pools for panels (b, c); panel (c) further restricts event
start time, leaving $792$ and $353$ examples, respectively.
\begin{figure}[h]
    \centering
    \captionsetup[subfigure]{justification=centering,singlelinecheck=true}
    \begin{subfigure}[t]{0.325\linewidth}
        \centering
        \vspace{0pt}
        \includegraphics[width=\linewidth]{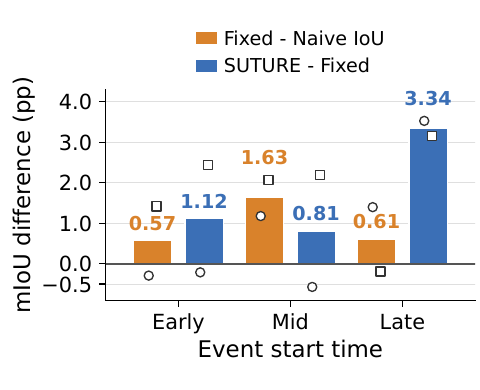}
        \caption{Final-answer mIoU gains.}
    \end{subfigure}\hfill
    \begin{subfigure}[t]{0.325\linewidth}
        \centering
        \vspace{0pt}
        \includegraphics[width=\linewidth]{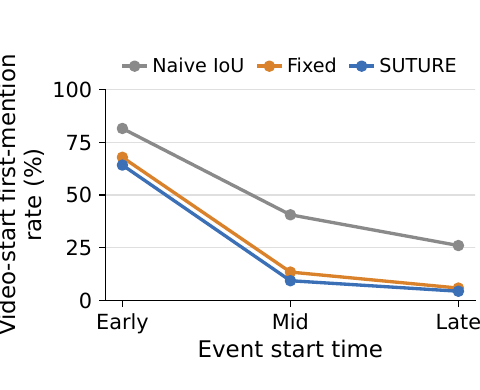}
        \caption{Video-start anchoring.}
    \end{subfigure}\hfill
    \begin{subfigure}[t]{0.325\linewidth}
        \centering
        \vspace{0pt}
        \includegraphics[width=\linewidth]{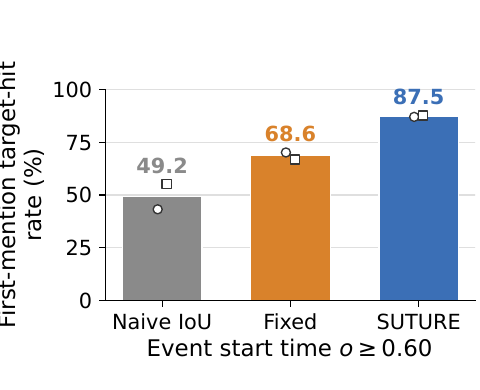}
        \caption{Late-event target hits.}
    \end{subfigure}
    \caption{\textbf{Verifier components by event start time.}
    Fixed uses $a=0.4$. Bars in (a, c) and curves in (b) average Charades-STA
    and QVHighlights equally. Numeric labels give the bar means; open circles
    and squares show Charades-STA and QVHighlights, respectively, in (a, c).
    Panel (c) uses events with $o\ge0.60$.}
    \label{fig:adaptive_mechanism}
\end{figure}

Let $s_{\mathrm{GT}}$ denote the annotated event's start time,
$s_{\mathrm{first}}$ the start time of the first temporal mention,
and $D_V$ the video duration.
Using normalized event start time $o=s_{\mathrm{GT}}/D_V$,
panels (a, b) use Early ($o<0.15$), Mid ($0.15\le o<0.45$),
and Late ($o\ge0.45$) bins. Panel (a) reports mIoU differences in percentage points.
Video-start first mentions in (b) satisfy
$s_{\mathrm{first}}/D_V<0.02$. A target hit in (c) means strictly positive overlap
with any annotated target interval, without an IoU threshold.

The additional mIoU gain from \method{} over the fixed control is largest
in the Late bin: $3.53$ points on Charades-STA and $3.15$ on QVHighlights,
or $3.34$ points with equal benchmark weights.

Most of the reduction in video-start anchoring occurs from Naive IoU to the
fixed control. The further reduction with \method{} is smaller, and the
two methods have the same Late-bin rate on QVHighlights ($2.6\%$).
Their first-mention target-hit rates differ more clearly for events with
$o\ge0.60$: \method{} reaches $87.1\%$ versus $70.2\%$ on Charades-STA
and $87.8\%$ versus $66.9\%$ on QVHighlights.
Reduced video-start anchoring does not by itself imply a target hit.
These observations do not establish that first-mention behavior causes the
grounding gains.

\subsection{Generated response length}
\label{app:response_length}

Table~\ref{tab:response_length} reports output token counts for the
Qwen2.5-VL-7B-Instruct policies. The TaRO and \method{} counts use
the same evaluation outputs as Table~\ref{tab:main_results}.
We count tokens with the Qwen2.5-VL-7B-Instruct tokenizer and report
total output length alongside separate counts for the reasoning
trace and final answer.
All three variants use the evaluation settings in Table~\ref{tab:train_eval_settings},
including a response limit of 1,024 tokens, a budget of 4,096 visual tokens, a model
context limit of 16,384 tokens, and video sampling at 2 fps. We measure output tokens
only.

\begin{table}[H]
\caption{\textbf{Generated response length.}
$n$ is the number of evaluated responses. P90 denotes the 90th
percentile of total output length. Reasoning and answer columns
report mean token counts for the respective response portions.
At limit is the percentage of responses reaching 1,024
output tokens.}
\label{tab:response_length}
\centering
\small
\begin{tabular}{lrrrrrrr}
\toprule
& & \multicolumn{3}{c}{Output tokens}
& \multicolumn{2}{c}{Mean tokens} & \\
\cmidrule(lr){3-5}\cmidrule(lr){6-7}
Method & $n$ & Mean & Median & P90 & Reasoning & Answer & At limit (\%) \\
\midrule
\multicolumn{8}{l}{\textit{Charades-STA}} \\
\method{} & 3,720 & \textbf{101.9} & 100 & 134 & 77.6 & 14.3 & 0.00 \\
TaRO & 3,720 & 148.9 & 148 & 182 & 124.7 & 13.2 & 0.00 \\
Naive IoU & 3,720 & 195.3 & 193 & 253 & 170.7 & 14.2 & 0.05 \\
\midrule
\multicolumn{8}{l}{\textit{QVHighlights}} \\
\method{} & 1,550 & \textbf{146.0} & 109 & 258 & 119.9 & 15.3 & 0.06 \\
TaRO & 1,550 & 192.5 & 181 & 249 & 167.2 & 14.2 & 0.06 \\
Naive IoU & 1,550 & 447.2 & 369 & 845 & 423.2 & 14.2 & 6.52 \\
\midrule
\multicolumn{8}{l}{\textit{ActivityNet Captions}} \\
\method{} & 15,933 & \textbf{149.3} & 113 & 257 & 123.6 & 14.9 & 0.21 \\
TaRO & 15,933 & 195.2 & 182 & 256 & 170.3 & 13.9 & 0.04 \\
Naive IoU & 15,933 & 443.1 & 335 & 1024 & 420.8 & 13.1 & 12.03 \\
\bottomrule
\end{tabular}
\end{table}

Relative to TaRO, \method{} reduces mean output length by $31.6\%$
on Charades-STA, $24.2\%$ on QVHighlights, and $23.5\%$ on ActivityNet,
without an explicit length penalty in the reward.
Mean reasoning trace length decreases by
$37.8\%$, $28.3\%$, and $27.4\%$ on Charades-STA, QVHighlights,
and ActivityNet, respectively. Mean final answer length is instead
about one token longer. Median output length is lower on all three
benchmarks, but P90 is slightly higher on QVHighlights and ActivityNet.
Naive IoU reaches the response limit on $6.52\%$ of QVHighlights
and $12.03\%$ of ActivityNet examples; its length statistics therefore
reflect this generation cap rather than unconstrained response lengths.

These measurements come from a single run at the checkpoint selected for each method,
without controlled training seeds. They quantify the number of generated tokens rather
than elapsed generation time and do not establish that
reduced anchoring to the video start causes shorter responses.

\clearpage
\section{Protocol for the Residual Gradient Diagnostic}
\label{app:update_geometry_metrics}

This appendix defines the analysis reported in
Figure~\ref{fig:prop2_semantics}.
Panels~\ref{fig:prop2-support} and~\ref{fig:prop2-preference}
use fixed \method{} rollout groups from the first visit, with model parameters
held fixed at step 350.
The diagnostic measures local responses to the structured correction
on these fixed rollouts.

Panel~\ref{fig:prop2-support} reports pooled Spearman alignment
between temporal support deficits and directional responses associated with each segment.
Panel~\ref{fig:prop2-preference} reports the response difference between rollouts with
positive and negative corrections, standardized within each group;
faint paired markers show retained groups.
Large markers summarize each metric separately for recovered and
unrecovered groups, and error bars give 95\% bootstrap intervals from 20,000
replicates, using rollout groups as clusters.
The analysis includes 16 recovered and 19 unrecovered rollout groups
drawn from severe Naive IoU failures and retained under the criteria below.
Both diagnostic directions use score vectors normalized by response length
and the same selected model parameters: those of the language model head and last two decoder layers.

\subsection{Group selection for the gradient analysis}

We consider training instances observed three times under both
\method{} and the Naive IoU control.
The first visit denotes each prompt's first encounter during training.
We first restrict to severe Naive IoU failures, defined by a
mean IoU of at most $0.1$ across the group at the final visit.
Within this set, an instance is \emph{recovered} when its mean IoU across the \method{}
group at the final visit is at least $0.3$, and \emph{unrecovered}
when that value is at most $0.1$; instances with intermediate
\method{} outcomes are excluded from this comparison.

Before gradient computation, we match 25 recovered and 25 unrecovered
instances using Naive IoU scores at the first and final visits and the log of target duration.
Excluding groups with a degenerate direction of the structured correction
leaves 16 recovered and 19 unrecovered rollout groups,
containing 54 and 52 valid support bins, respectively.
Thus, the final gradient analysis is conditional on severe
Naive IoU failure, retrospective classification as recovered or
unrecovered, and exclusion of degenerate directions of the structured correction.
Because filtering need not preserve the original matched pairs,
uncertainty is estimated by bootstrapping rollout groups rather than matched pairs.

\subsection{Fixed model parameters and score vectors}

Each of the $35$ instances retained for the gradient analysis contributes its fixed
\method{} rollout group of $G=8$ responses from the first visit.
We evaluate those responses at
$\theta_\star=\theta_{350}$ in evaluation mode with dropout disabled,
using float32 arithmetic and the video and metadata preprocessing used for training.

Let $\mathcal S$ denote the selected model parameters: those of the
language model head and last two decoder layers. The same model parameters are used for
both diagnostic directions.

For response $y_i$, define the score vector summed over tokens and its counterpart
normalized by response length as
\begin{equation}
\mathbf s_i
=
\nabla_{\theta_{\mathcal S}}
\sum_t
\log\pi_{\theta_\star}
(y_{i,t}\mid x,y_{i,<t}),
\qquad
\mathbf z_i
=
\frac{\mathbf s_i}{|y_i|}.
\end{equation}

The replay stores the Gram matrix of scores summed over tokens and converts it to a
matrix normalized by response length. Both diagnostic panels in
Figure~\ref{fig:prop2_semantics} use $K^{(z)}$:
\begin{equation}
K^{(s)}_{ij}
=
\mathbf s_i^\top\mathbf s_j,
\qquad
K^{(z)}_{ij}
=
\frac{K^{(s)}_{ij}}
{|y_i||y_j|}.
\end{equation}

\Needspace{10\baselineskip}
\subsection{Diagnostic directions and first-order responses}

The diagnostic IoU direction normalizes temporal IoU alone:
\begin{equation}
\widehat A_i^{\mathrm{IoU}}
=
\frac{
R_i^{\mathrm{IoU}}-\bar R^{\mathrm{IoU}}
}{
\widehat\sigma_{\mathrm{IoU}}+\delta
},
\qquad
\mathbf g_{\mathrm{IoU}}
=
\frac1G
\sum_i
\widehat A_i^{\mathrm{IoU}}
\mathbf z_i,
\end{equation}
where $\bar R^{\mathrm{IoU}}$ and
$\widehat\sigma_{\mathrm{IoU}}$ denote the mean and standard deviation of the temporal
IoU scores within the group.

For the structured correction, we use the same centered reward
correction introduced in Section~\ref{sec:theory}.
For each realized rollout group,
\begin{equation}
\chi_i
:=
r_i^{\mathrm S}-r_i^{\mathrm{IoU}}
=
R_i^{\mathrm S}-R_i^{\mathrm{IoU}}
=
\beta_{\mathcal P}\xi_i,
\qquad
\widetilde\chi_i
=
\chi_i-\bar\chi,
\end{equation}
where
\begin{equation}
\bar\chi
=
\frac1G\sum_{j=1}^{G}\chi_j.
\end{equation}
The total and temporal reward differences are equal
because the controlled variants use the same format reward for each
rollout.

We define the diagnostic direction for the structured correction as
\begin{equation}
\label{eq:diagnostic_corr_direction}
\mathbf g_{\mathrm{corr}}
=
\frac1G
\sum_i
\widetilde\chi_i\mathbf z_i.
\end{equation}
This computes the correction direction from
Equation~\ref{eq:gradient_decomp_body} for the fixed rollouts,
using only the selected model parameters.
The factor $\beta_{\mathcal P}$ is already contained in
$\widetilde\chi_i$; we do not additionally multiply the diagnostic
direction by $\alpha_{\mathcal P}$.

For $d\in\{\mathrm{IoU},\mathrm{corr}\}$, define the rollout response as
\begin{equation}
\label{eq:fig3_response}
q_i^{(d)}
=
\mathbf z_i^\top\mathbf g_d.
\end{equation}
This quantity is the directional derivative of the
log-likelihood normalized by response length in the analyzed parameter
coordinates.

The IoU reference direction uses temporal IoU rather than the
implemented total reward.
The comparison therefore isolates the temporal IoU direction and
the structured correction direction rather than reconstructing the
complete actor update.

\subsection{Support alignment}

We partition the implemented temporal representation at rollout interval boundaries
and use valid target segments with rollout coverage.
For segment $k$ in group $g$, let $c_{gk}$ be the fraction of
rollouts covering the segment and let $\bar c_g$ be the mean support
over its valid segments.

Define the support deficit and the mean directional response of
covering rollouts as
\begin{equation}
\delta_{gk}
=
\bar c_g-c_{gk},
\qquad
Q_{gk}^{(d)}
=
\frac{1}{|\mathcal I_{gk}|}
\sum_{i\in\mathcal I_{gk}}
q_i^{(d)},
\end{equation}
where $\mathcal I_{gk}$ is the set of rollouts covering segment $k$.

We pool segment observations separately for recovered and unrecovered
groups and define support alignment as
\begin{equation}
\label{eq:fig3_support_alignment}
\rho_d
=
\operatorname{Spearman}
\left(
\{\delta_{gk}\},
\{Q_{gk}^{(d)}\}
\right).
\end{equation}
Positive values indicate that rollouts covering segments with relatively low support
have larger directional responses in the
pooled analysis.

\begin{table}[H]
\caption{Support alignment in
Figure~\ref{fig:prop2-support}.
Intervals are 95\% bootstrap intervals with rollout groups as clusters.}
\label{tab:fig3_support_alignment}
\centering
\small
\begin{tabular}{llcc}
\toprule
Final outcome & Direction & $\rho_d$ & 95\% interval \\
\midrule
Recovered
    & $\mathbf g_{\mathrm{IoU}}$
    & $-0.056$
    & $[-0.393,\ 0.327]$ \\
Recovered
    & $\mathbf g_{\mathrm{corr}}$
    & $\mathbf{0.614}$
    & $[0.451,\ 0.798]$ \\
Unrecovered
    & $\mathbf g_{\mathrm{IoU}}$
    & $0.156$
    & $[-0.161,\ 0.514]$ \\
Unrecovered
    & $\mathbf g_{\mathrm{corr}}$
    & $\mathbf{0.513}$
    & $[0.292,\ 0.783]$ \\
\bottomrule
\end{tabular}
\end{table}

\subsection{Residual-sign preference gap}

For each retained rollout group, define
\begin{equation}
\label{eq:fig3_preference_gap}
\operatorname{Gap}_g(\mathbf g_d)
=
\frac{
\operatorname{mean}_{\widetilde\chi_i>0}
q_i^{(d)}
-
\operatorname{mean}_{\widetilde\chi_i<0}
q_i^{(d)}
}{
\operatorname{SD}_i(q_i^{(d)})
}.
\end{equation}

We average these standardized gaps for individual groups separately over
recovered and unrecovered groups.
Positive values indicate larger directional responses for rollouts with positive
corrections than for those with negative corrections on average.
Thus, the residual-sign split is determined by the sign of the
centered structured reward correction $\widetilde\chi_i$, rather
than by an independently chosen support threshold.

\begin{table}[H]
\caption{Residual-sign preference gap in
Figure~\ref{fig:prop2-preference}.
Intervals are 95\% bootstrap intervals with rollout groups as clusters.}
\label{tab:fig3_preference_gap}
\centering
\small
\begin{tabular}{llcc}
\toprule
Final outcome & Direction & Mean gap (SD) & 95\% interval \\
\midrule
Recovered
    & $\mathbf g_{\mathrm{IoU}}$
    & $-0.326$
    & $[-0.983,\ 0.355]$ \\
Recovered
    & $\mathbf g_{\mathrm{corr}}$
    & $\mathbf{1.664}$
    & $[1.500,\ 1.847]$ \\
Unrecovered
    & $\mathbf g_{\mathrm{IoU}}$
    & $-0.066$
    & $[-0.605,\ 0.467]$ \\
Unrecovered
    & $\mathbf g_{\mathrm{corr}}$
    & $\mathbf{1.903}$
    & $[1.729,\ 2.103]$ \\
\bottomrule
\end{tabular}
\end{table}

\subsection{Uncertainty, scaling, and scope}

We use 20,000 bootstrap replicates with the rollout group as the resampling unit, so
segments and responses from the
same group are not treated as independent.

The 95\% intervals for the
$\mathbf g_{\mathrm{corr}}$-minus-$\mathbf g_{\mathrm{IoU}}$
contrast are $[0.234,1.100]$ and $[0.046,0.723]$ for support
alignment in the recovered and unrecovered groups, respectively.
For the mean preference gap, the corresponding intervals are
$[1.328,2.663]$ and $[1.446,2.516]$.
These intervals characterize uncertainty within the retained
diagnostic population rather than variation across independent
training runs.

The preference gap standardized within each group is invariant to a
positive scaling shared within a group.
Pooled support alignment is not generally invariant to
positive scaling that varies across groups, since such scaling can change
ranks across groups.
Accordingly, we report $\mathbf g_{\mathrm{corr}}$ exactly as defined
in Equation~\ref{eq:diagnostic_corr_direction}, without applying the
additional factor $\alpha_{\mathcal P}$ from
Equation~\ref{eq:gradient_decomp_body}.

This experiment evaluates fixed past responses with model parameters
held fixed, using only the selected model parameters.
It characterizes local responses to the structured correction rather
than reconstructing the complete actor update.

\clearpage
\section{Qualitative Examples of Tracking Event Start Times}
\label{app:qualitative_activitynet}

\paragraph{Setup.}
Figures~\ref{fig:qualitative-onset-1} and~\ref{fig:qualitative-onset-2}
complement Section~\ref{sec:onset_tracking} with individual ActivityNet
examples. Each panel shows the ground truth interval, the first well-formed
temporal interval in the reasoning trace, the final prediction, and abbreviated
response text. Dashed outlines denote first mentions and solid bars denote final
answers. \method{}, TaRO, and plain IoU share the same backbone and training
setup; Time-R1 is shown only as a reference because its training
configuration differs.

\paragraph{Selection.} The four examples use a fixed ActivityNet selection
rule: the target starts after $40\%$ of the video, the first \method{} mention
overlaps the target, the first TaRO mention starts within the first $5\%$, and
\method{} exceeds every baseline by at least $0.60$ IoU. We additionally draw
examples from nonoverlapping bands of video duration and target length. The rule is
intended to expose clear instances of the behavior in
Section~\ref{sec:onset_tracking}; the quantitative claims rely on the full
evaluation sets rather than these selected panels.

\begin{figure}[H]
    \centering
    \begin{subfigure}[t]{\linewidth}
        \centering
        \includegraphics[width=\linewidth]{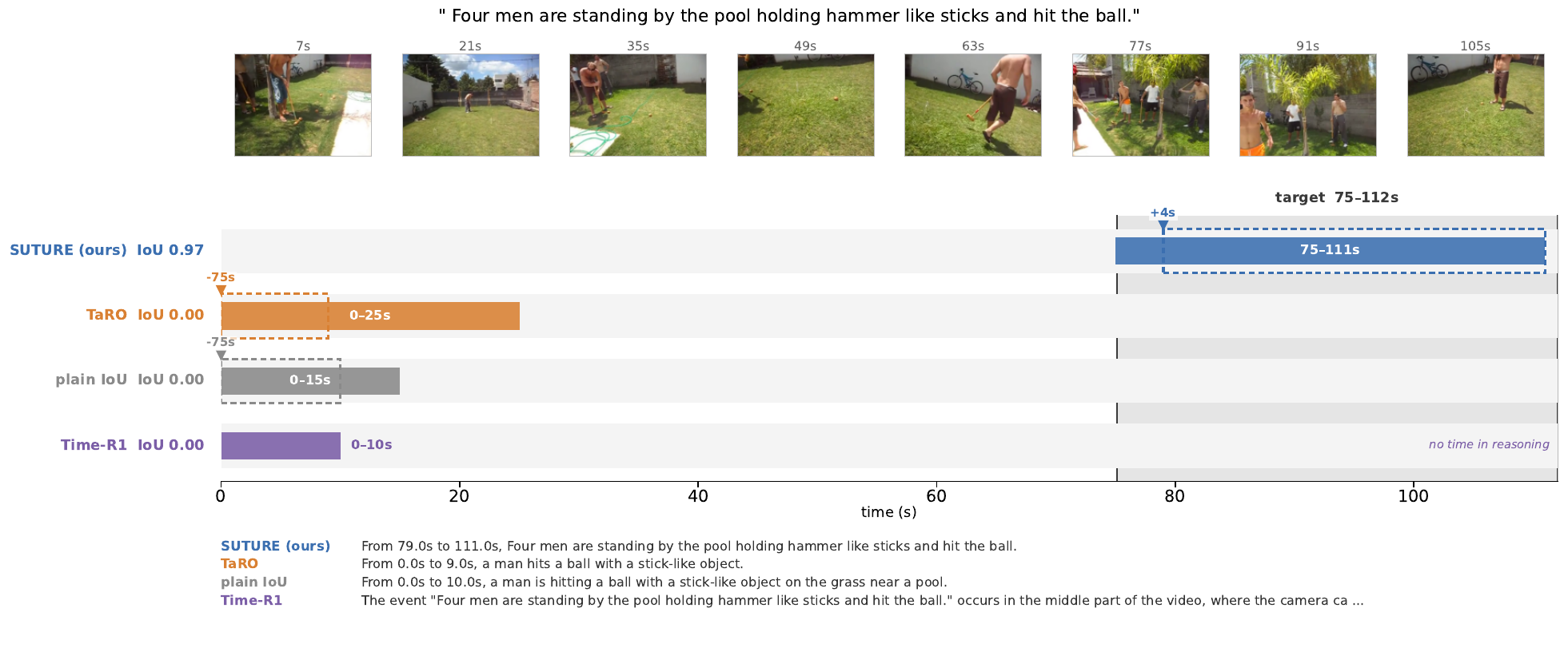}
        \caption{\emph{``Four men are standing by the pool holding hammer like sticks and hit the ball.''} The ground truth interval is 75--112s of a 112s video.}
        \label{fig:qualitative-onset-pool}
    \end{subfigure}

    \vspace{0.8em}
    \begin{subfigure}[t]{\linewidth}
        \centering
        \includegraphics[width=\linewidth]{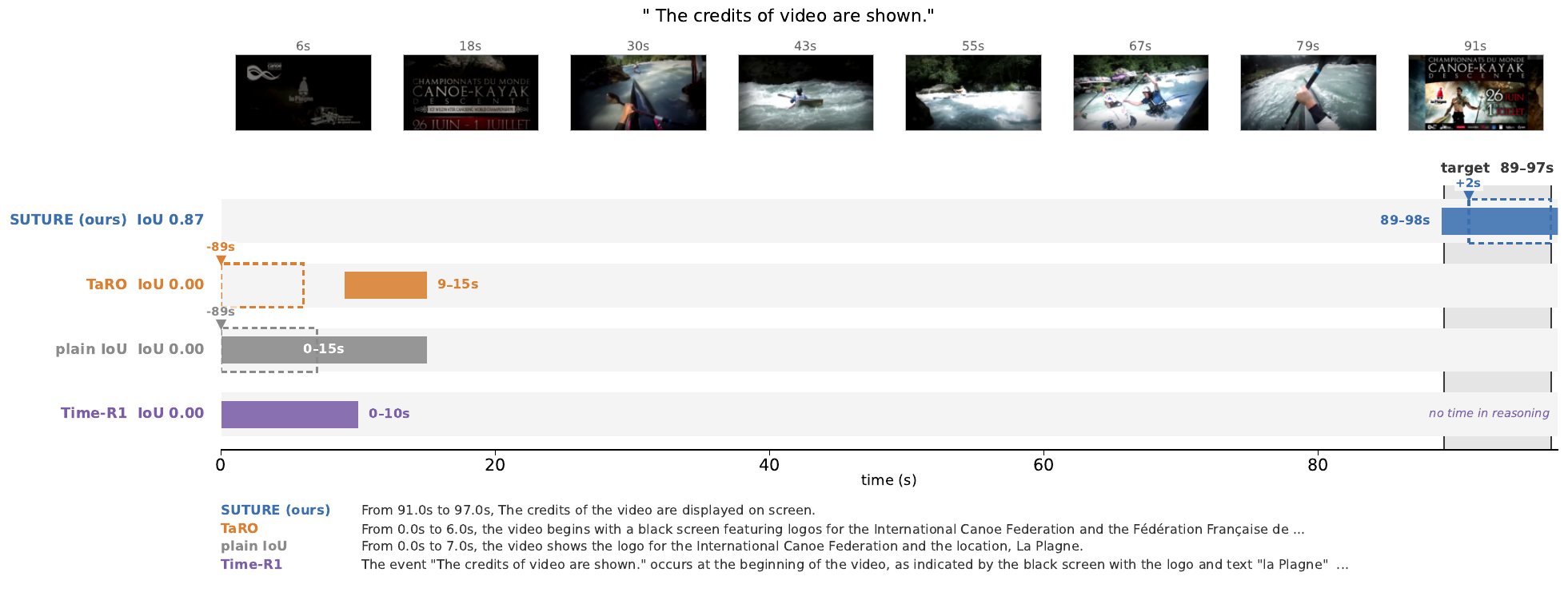}
        \caption{\emph{``The credits of video are shown.''} The ground truth interval is 89--97s of a 98s video, and a similar title card appears at the start of the video.}
        \label{fig:qualitative-onset-credits}
    \end{subfigure}
    \caption{\textbf{Tracking event start times during reasoning on ActivityNet (I).} Dashed
    outlines mark first temporal mentions and solid bars mark final answers. In both
    examples \method{} opens its reasoning inside the target interval, while TaRO and
    plain IoU open near the video start.}
    \label{fig:qualitative-onset-1}
\end{figure}

\clearpage
\paragraph{Interpretation.}
The examples visualize the same pattern of early first-mention start times seen in the aggregate
analysis. In the first pair, \method{} moves its initial temporal hypothesis to
the target interval later in the video while TaRO and plain IoU remain near the video start.
The second pair also shows that the location of the first mention and the final answer need
not move together: comparison policies can partially recover in the final
prediction despite opening their temporal reasoning substantially earlier.

\begin{figure}[H]
    \centering
    \begin{subfigure}[t]{\linewidth}
        \centering
        \includegraphics[width=\linewidth]{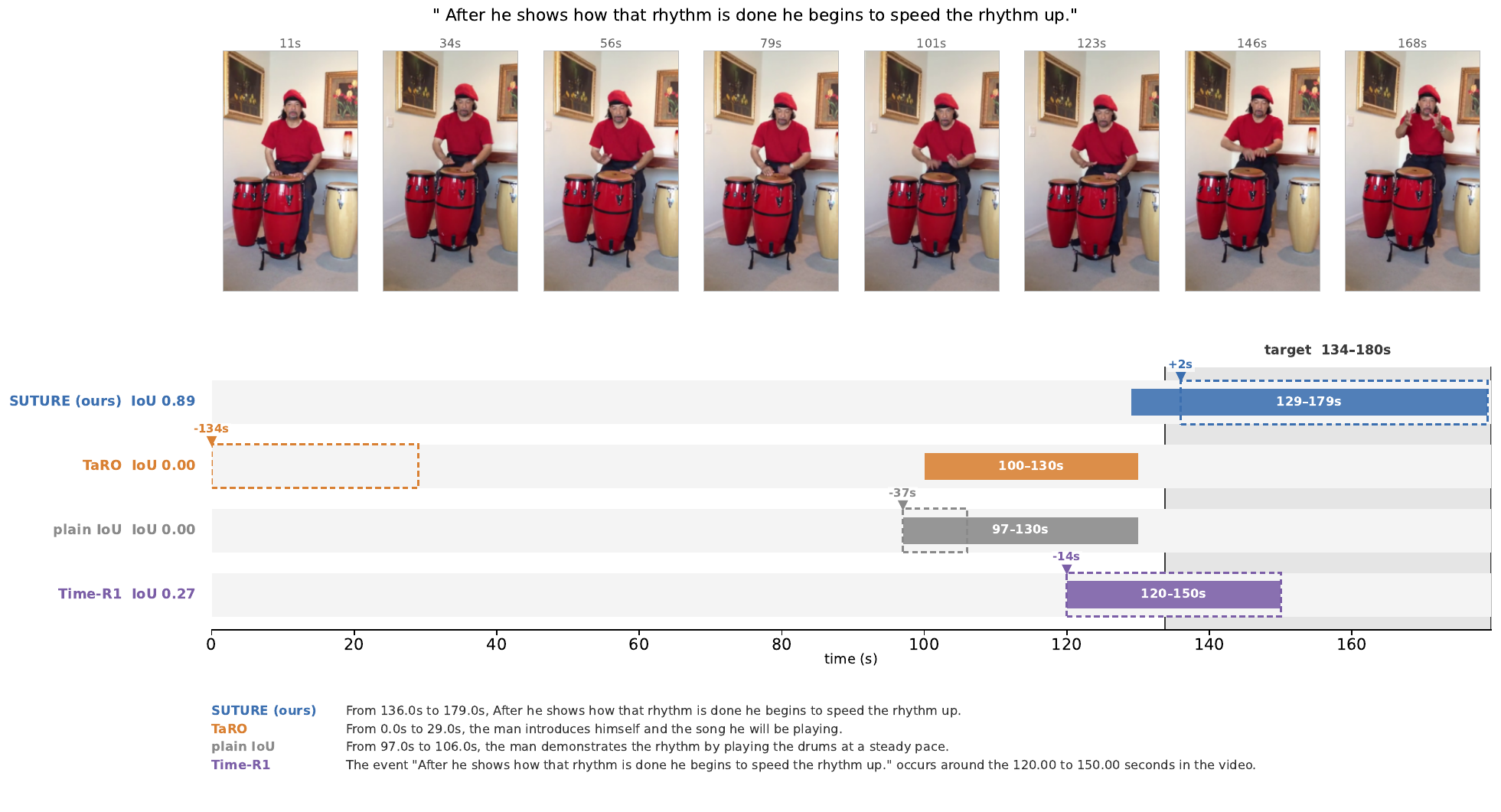}
        \caption{\emph{``After he shows how that rhythm is done he begins to speed the rhythm up.''} The ground truth interval is 134--180s of a 180s video.}
        \label{fig:qualitative-onset-drums}
    \end{subfigure}

    \vspace{0.8em}
    \begin{subfigure}[t]{\linewidth}
        \centering
        \includegraphics[width=\linewidth]{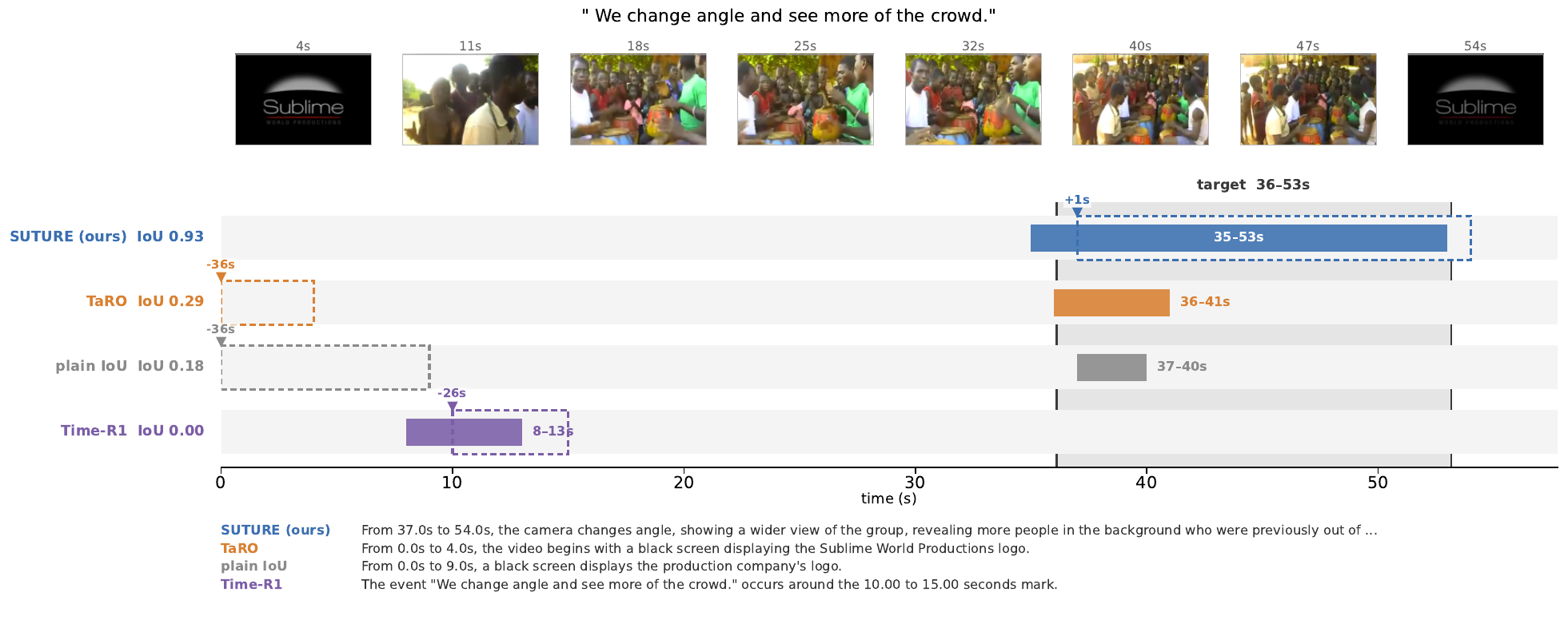}
        \caption{\emph{``We change angle and see more of the crowd.''} The ground truth interval is 36--53s of a 58s video.}
        \label{fig:qualitative-onset-crowd}
    \end{subfigure}
    \caption{\textbf{Tracking event start times during reasoning on ActivityNet (II).} These
    examples show cases in which comparison policies partially recover in the final
    answer despite first mentioning substantially earlier temporal regions.}
    \label{fig:qualitative-onset-2}
\end{figure}

\clearpage
\endgroup

\end{document}